\documentclass[journal]{IEEEtran}

\IEEEoverridecommandlockouts

\usepackage{cite}
\usepackage{amsmath,amssymb,amsfonts}
\usepackage{algorithmicx}
\usepackage{algpseudocode}
\usepackage{graphicx}
\usepackage{textcomp}
\usepackage{xcolor}
\def\BibTeX{{\rm B\kern-.05em{\sc i\kern-.025em b}\kern-.08em
    T\kern-.1667em\lower.7ex\hbox{E}\kern-.125emX}}

\usepackage{amsmath}
\usepackage{amsfonts}
\usepackage{color}
\usepackage{soul}
\usepackage{cite}
\usepackage{graphicx}
\usepackage{algorithm}
\usepackage{mathtools}
\usepackage{bm}
\usepackage{booktabs}
\usepackage{textcomp}
\usepackage{nicematrix}

\usepackage{pgf}
\usepackage{printlen}

\usepackage{hyperref}

\newtheorem{theorem}{Theorem}
\newtheorem{corollary}{Corollary}
\newtheorem{assumption}{Assumption}
\newtheorem{lemma}{Lemma}
\newtheorem{remark}{Remark}
\newtheorem{proposition}{Proposition}
\newtheorem{defn}{Definition}

\newenvironment{proof}{\paragraph*{Proof}}{\hfill$\square$}

\DeclareMathSymbol{\shortminus}{\mathbin}{AMSa}{"39}

\DeclareMathOperator*{\argmin}{\mathrm{argmin}}

\DeclareMathOperator*{\nnz}{\mathrm{nnz}}

\newcommand{\real}{\mathbb{R}}

\algnewcommand{\LineComment}[1]{\State \(\triangleright\) #1}

\usepackage[dvipsnames]{xcolor}
\usepackage{marginnote}
\usepackage{xstring}

\newif\ifreviewermode
\reviewermodefalse

\newcommand{\revres}[2]{%
    \ifreviewermode
        \IfEqCase{#1}{%
            {1}{\textcolor{brown}}%
            {2}{\textcolor{magenta}}%
            {3}{\textcolor{olive}}%
            {4}{\textcolor{teal}}%
            {5}{\textcolor{LimeGreen}}%
            {6}{\textcolor{violet}}%
        }[\PackageError{revres}{Undefined revres option: #1}{}]{#2}%
    \else
        {#2}%
    \fi
}%
\newcommand{\revmargin}[3][false]{%
    \ifreviewermode
        \def\placemarginnote{%
            \IfEqCase{#2}{%
                {1}{\textcolor{brown}}%
                {2}{\textcolor{magenta}}%
                {3}{\textcolor{olive}}%
                {4}{\textcolor{teal}}%
                {5}{\textcolor{LimeGreen}}%
                {6}{\textcolor{violet}}%
            }[\PackageError{revmargin}{Undefined revmargin option: #2}{}]%
            {\marginnote{#3}}%
        }%
        \IfStrEq{#1}{true}{%
            {
                \reversemarginpar
                \placemarginnote
                \normalmarginpar 
            }%
        }{%
            \placemarginnote
        }%
    \else
    \fi
}

\usepackage[normalem]{ulem}

\newif\ifanonymousmode
\anonymousmodefalse

\newcommand{\anon}[2]{%
    \ifanonymousmode
        {#1}%
    \else
        {#2}%
    \fi
}%

\begin{document}

\title{Hybrid System Planning using a Mixed-Integer ADMM Heuristic and Hybrid Zonotopes}

\author{Joshua A. Robbins, Andrew F. Thompson, Jonah J. Glunt, Herschel C. Pangborn
\thanks{All authors are with the Department of Mechanical Engineering, The Pennsylvania State University, University Park, PA 16802 USA (e-mail: {\tt\small jrobbins@psu.edu, thompson@psu.edu, jglunt@psu.edu, hcpangborn@psu.edu}).}
}

\maketitle

\begin{abstract}
Embedded optimization-based planning for hybrid systems is challenging due to the use of mixed-integer programming, which is computationally intensive and often sensitive to the specific numerical formulation. To address that challenge, this article proposes a framework for motion planning of hybrid systems that pairs hybrid zonotopes---an advanced set representation---with a new alternating direction method of multipliers (ADMM) mixed-integer programming heuristic. A general treatment of piecewise affine (PWA) system reachability analysis using hybrid zonotopes is presented and extended to formulate optimal planning problems. Sets produced using the proposed identities have lower memory complexity and tighter convex relaxations than equivalent sets produced from preexisting techniques. The proposed ADMM heuristic makes efficient use of the hybrid zonotope structure. For planning problems formulated as hybrid zonotopes, the proposed heuristic achieves improved convergence rates as compared to state-of-the-art mixed-integer programming heuristics. 
The proposed methods for hybrid system planning on embedded hardware are experimentally applied in a combined behavior and motion planning scenario for autonomous driving.
\end{abstract}

\section{Introduction}

Dynamic systems with both continuous and discrete behavior are called hybrid systems. These are often used to model systems with multiple dynamic modes. 
In robotics, hybrid systems arise when there are contact dynamics, as in cases of walking robots or manipulators~\cite{johnson2016hybrid}. Hybrid systems can also appear in combined task and motion planning~\cite{garrett2021integrated} or, analogously, integrated behavior and motion planning for autonomous driving~\cite{quirynen2024real, esterle2020optimal, danielson2017constraint}.

This paper is concerned with motion planning problems for linear systems with disjoint constraints and hybrid systems where each dynamic mode is an affine system. This class of hybrid systems is equivalently described as piecewise affine (PWA) systems, mixed-logical dynamical (MLD) systems, or discrete hybrid automata (DHA)~\cite{torrisi2004hysdel}. 
PWA systems in particular are widely used in robotics, with examples including legged robots~\cite{buchan2013automatic}, skidding UGVs~\cite{benine2012piecewise}, and robotic manipulation systems~\cite{han2020local, hogan2020reactive}. 
Planning problems for PWA systems are challenging to solve, with computational efficiency very sensitive to the specific numerical formulation~\cite{marcucci2019mixed}.

\begin{figure}[t]
    \centering
    \input{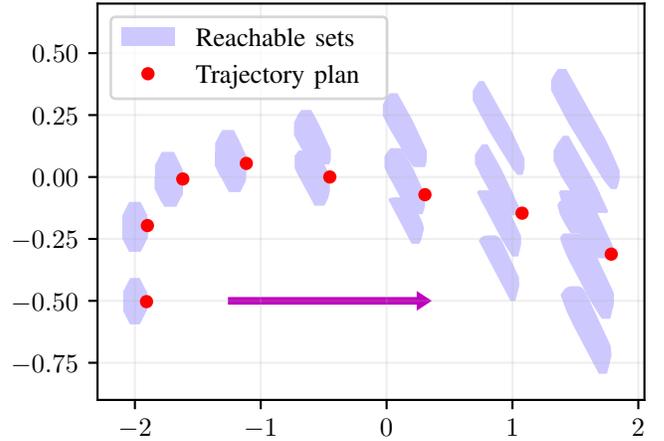}
    \caption{Hybrid zonotope reachability analysis is used to construct planning problems for PWA systems. Then an ADMM-based mixed integer programming heuristic efficiently finds feasible, and often nearly optimal, solutions.}
    \label{fig:conceptual-hybrid}
\end{figure}

\subsection{Gaps in the Literature}

Reachability analysis---i.e., rigorously quantifying the states that a dynamic system can achieve---is a powerful tool for formulating planning problems. Advantages include 
the ability to account for disturbances~\cite{seo2022real, liu2024refine}, and
favorable numerical structure and properties~\cite{robbins2025sparsity}.
Zonotopes~\cite{girard2005reachability} and constrained zonotopes~\cite{scott2016constrained} are advanced set representations widely used to compute reachable sets for linear systems. Hybrid zonotopes~\cite{bird2023hybrid} are an extension of constrained zonotopes that can be used to compute reachable sets for hybrid systems and represent non-convex constraints. A general treatment of MLD system reachability analysis using hybrid zonotopes was provided in~\cite{bird2023hybrid}. Modeling with MLD systems generally requires a modeling language like HYSDEL~\cite{torrisi2004hysdel}, \revmargin{3}{R3.1}which inhibits applicability to systems where the model is not known \textit{a priori}. This can occur in combined behavior and motion planning where the allowed behaviors are environment-dependent, and when using data-based models as in~\cite{buchan2013automatic}. Hybrid zonotope-based reachability analysis for some specific DHA systems was also considered in~\cite{siefert2024reachability}. A comprehensive treatment of hybrid zonotope-based reachability analysis for PWA systems has not yet been provided.

Hybrid zonotopes are a mixed-integer set representation, and as such, planning problems formulated using hybrid zonotope reachability analysis are mixed-integer programs (MIPs). The use of MIPs in motion planning is an active area of research. While MIPs can model non-convex and disjoint constraints, solving them is NP-hard in general.
Typically, solution approaches are based on branch-and-bound or branch-and-cut methods, often using powerful commercial solvers such as Gurobi or CPLEX~\cite{ioan2021mixed}. Problem-tailored solvers have been proposed as well~\cite{quirynen2024real, robbins2024mixed}.
See~\cite{ioan2021mixed} for a review of MIPs in motion planning.

For mixed-integer convex programs, branch-and-bound and branch-and-cut methods are guaranteed to converge to the optimal solution but can be computationally intensive, limiting utility for embedded applications. In particular, memory requirements can be prohibitive for complex problems due to queue data structure needed to store sub-problems~\cite{karamanov2006branch}. These challenges have led some to consider mixed-integer programming heuristics for embedded applications. The most widely-used mixed-integer heuristic is the feasibility pump~\cite{fischetti2005feasibility}.
This is effective at finding feasible solutions in practice, and is used as a primal heuristic in most high-performance MIP solvers~\cite{berthold2019ten}.
In contrast with branching-based methods, the feasibility pump is purely iterative and does not require any sub-problem queue, alleviating memory challenges. Because the method is heuristic, there are no convergence guarantees. A disadvantage of the feasibility pump is that each iteration requires the solution of a linear program (LP).

For embedded applications, MIP heuristics based on accelerated dual gradient projection~\cite{naik2017embedded} and the alternating direction method of multipliers (ADMM)~\cite{takapoui2020simple, alavian2017improving, liu2022distributed, diamond2016general, kanno2018alternating} have also been proposed.
Like the feasibility pump, ADMM is an iterative method and does not require a sub-problem queue. ADMM-based heuristics are simple to implement, and iterations typically have much lower computational complexity than the feasibility pump. Feasible MIP solutions with low objective function values are often found, motivating the use of ADMM heuristics as a substitute for traditional branching-based solution methods~\cite{takapoui2020simple}.
Performance can be inconsistent however, and most implementations require running the heuristic several times with different randomly generated initial iterates. For complex MIPs, this amounts to sampling in a high-dimensional space, and many samples may be required to find an initial condition such that the algorithm produces a feasible solution. This challenge is particularly significant for hybrid zonotope-based problem formulations, as the feasible set is implicitly defined as an affine mapping of a high-dimensional factor space. Despite ADMM-based heuristics having been extensively proposed for embedded applications, experimental implementations of these heuristics are lacking in the literature.

\subsection{Contributions}

This paper develops an efficient framework for formulating and solving planning problems for PWA systems on embedded hardware. 
\revmargin{3}{R3.3}\revres{3}{Example application areas include, but are not limited to, autonomous mobile robot planning and autonomous driving.}
This extends \anon{related work}{our prior work} that considered only linear systems and convex constraint sets~\cite{robbins2025sparsity} to address PWA systems and non-convex constraint sets requiring solution of MIPs. 

The key contributions of this paper are: 
(1) A general treatment of hybrid zonotope-based reachability for PWA systems is presented, and the resulting \revmargin{2}{R2.4}\revres{2}{set representations} are shown to have lower memory complexity and tighter convex relaxations than those produced using equivalent reachability calculations for MLD systems. These reachability calculations are extended to a lifted, constrained form to facilitate planning problem formulations. 
The proposed calculations exclusively use closed-form identities, facilitating online implementation for robotic systems.
(2) A novel ADMM-based MIP heuristic makes efficient use of the hybrid zonotope structure, and is applicable to any problem represented as a hybrid zonotope, including PWA system planning problems and planning problems for linear systems with disjoint constraints.
To improve the reliability of the proposed heuristic, perturbation and restart procedures inspired by the feasibility pump are employed. When compared to a state-of-the-art ADMM heuristic, numerical results show a significant increase in the number of cases where a feasible solution is produced---up to 30 times in our testing---while retaining \revmargin{2}{R2.1}\revres{2}{low computation times and memory usage}. 
(3) The proposed framework is employed for combined behavior and motion planning of an autonomous vehicle in a laboratory experiment, demonstrating its applicability to resource-limited embedded systems. 

Additionally, the appendix provides previously unpublished hybrid zonotope union identities. 
In particular, the \textit{condensed union} identity is leveraged throughout this paper to produce low-complexity hybrid zonotope planning problem formulations.

\revmargin{2}{R2.1}\revres{2}{The proposed framework is most applicable to robot planning problems where compute time and/or memory usage are constrained. Given the optimization-based formulation, nearly optimal solutions are often produced in practice. However, because the proposed heuristic lacks convergence guarantees, this framework should only be used in contexts where fall-back solutions or behaviors are available, or pre-converged trajectories that approximate a feasible solution can be used.}

\subsection{Outline}
The remainder of this paper is organized as follows. Sec.~\ref{sec:prelims} provides preliminary information. Sec.~\ref{sec:hz_reach} develops reachability calculations for PWA systems. Sec.~\ref{sec:admm-fp} proposes a novel ADMM-based MIP heuristic that uses perturbation procedures inspired by the feasibility pump. The developed hybrid zonotope reachability and ADMM heuristic framework is applied to several numerical examples in Sec.~\ref{sec:numerical-results}, and is experimentally implemented in Sec.~\ref{sec:experiment}. Sec.~\ref{sec:conclusion} concludes the paper. 
\section{Preliminaries} \label{sec:prelims}

\subsection{Notation}
Unless otherwise stated, scalars are denoted by lowercase letters, vectors by boldface lowercase letters, matrices by uppercase letters, and sets by calligraphic letters. $\mathbf{0} = \begin{bmatrix} 0 & \cdots & 0 \end{bmatrix}^T$ and $\mathbf{1} = \begin{bmatrix} 1 & \cdots & 1 \end{bmatrix}^T$ denote vectors consisting entirely of zeroes and ones, respectively, of appropriate dimensions. $I$ denotes the identity matrix. Projection of a vector $\mathbf{v}$ onto a set $\mathcal{C}$ is denoted as $\pi_{\mathcal{C}}(\mathbf{v})$. The indicator function for a set $\mathcal{C}$ is defined as $$I_{\mathcal{C}}(\mathbf{x}) = \begin{cases} {0, \; \mathbf{x} \in \mathcal{C}} \;, \\ {\infty, \; \text{otherwise}} \;. \end{cases}$$ The function $\mathrm{rand}(a,b)$ returns a random value drawn uniformly from the interval $[a,b]$. A diagonal matrix with elements $s_{ii}$ is denoted as $\mathrm{diag}([s_{11}, s_{22}, ...])$, and similarly, a block diagonal matrix with blocks $S_{ii}$ is denoted as $\mathrm{blkdiag}([S_{11}, S_{22}, ...])$. The operation returning the number of nonzero elements in a matrix is denoted as $\mathrm{nnz}(\cdot)$. An $n$-sided regular zonotope inner-approximation of a circle of radius $r$ is denoted as $\mathcal{O}(r,n) \subset \real^2$.

\subsection{Zonotopes, Constrained Zonotopes, and Hybrid Zonotopes}

A set $\mathcal{Z}$ is a zonotope if $\exists G \in \real^{n \times n_G}$ and $\mathbf{c} \in \real^n$ such that~\cite{girard2005reachability}
\begin{equation}
    \mathcal{Z} = \left\{ G \bm{\xi} + \mathbf{c} \revmargin{3}{R3.2}~\middle|~ \bm{\xi} \in [-1,1]^{n_G} \right\} \;.
\end{equation}
Zonotopes are centrally symmetric convex polytopes.

Constrained zonotopes extend the definition of zonotopes to include equality constraints. A set $\mathcal{Z}$ is a constrained zonotope if there exists a generator matrix $G \in \real^{n \times n_G}$, center $\mathbf{c} \in \real^n$, constraint matrix $A \in \real^{n_C \times n_G}$, and constraint vector $\mathbf{b} \in \real^{n_C}$ such that
\begin{equation} \label{eq:conzono-definition}
    \mathcal{Z} = \left\{ G \bm{\xi} + \mathbf{c} \revmargin{3}{R3.2}~\middle|~ A \bm{\xi} = \mathbf{b},\; \bm{\xi} \in [-1,1]^{n_G} \right\} \;.
\end{equation}
Constrained zonotopes can represent any convex polytope \cite{scott2016constrained}.

Hybrid zonotopes extend constrained zonotopes by allowing for both continuous and binary factors, $\bm{\xi}_c$ and $\bm{\xi}_b$. A set $\mathcal{Z}$ is a hybrid zonotope if there exist continuous and binary generator matrices $G_c \in \real^{n \times n_{Gc}}$ and $G_b \in \real^{n \times n_{Gb}}$, center $\mathbf{c} \in \real^n$, continuous and binary constraint matrices $A_c \in \real^{n_C \times n_{Gc}}$ and $A_b \in \real^{n_C \times n_{Gb}}$, and constraint vector $\mathbf{b} \in \real^{n_C}$ such that
\begin{equation} \label{eq:hybzono-definition}
    \mathcal{Z} = \left\{ \begin{bmatrix} G_c & G_b \end{bmatrix} \begin{bmatrix} \bm{\xi}_c \\ \bm{\xi}_b \end{bmatrix} + \mathbf{c} \revmargin{3}{R3.2}~\middle|~ 
    \begin{matrix} 
    \begin{bmatrix} A_c & A_b \end{bmatrix} \begin{bmatrix} \bm{\xi}_c \\ \bm{\xi}_b \end{bmatrix} = \mathbf{b} \\ \bm{\xi}_c \in [-1,1]^{n_{Gc}}\\ \bm{\xi}_b \in \{-1, 1\}^{n_{Gb}} 
    \end{matrix} \right\} \;.
\end{equation}
Hybrid zonotopes represent unions of polytopes~\cite{bird2023hybrid}.

In the above definitions, $n$ is the dimension of the hybrid zonotope, $n_{Gc}$ is the number of continuous generators, $n_{Gb}$ is the number of binary generators, $n_G = n_{Gc} + n_{Gb}$ is the total number of generators, and $n_C$ is the number of equality constraints. The shorthand notation $\mathcal{Z} = \left\langle G_c, G_b, \mathbf{c}, A_c, A_b, \mathbf{b} \right\rangle$ is often used in place of \eqref{eq:hybzono-definition}. Similarly, $\left\langle G, \mathbf{c}, A, \mathbf{b} \right\rangle$ denotes a constrained zonotope, and $\left\langle G, \mathbf{c} \right\rangle$ denotes a zonotope. 

It is often useful to define hybrid zonotopes such that $\bm{\xi}_c \in [0,1]^{n_{Gc}}$ and $\bm{\xi}_b \in \{0, 1\}^{n_{Gb}}$. We say that hybrid zonotopes defined in this way are in \textit{01-form}, and that hybrid zonotopes defined using~\eqref{eq:hybzono-definition} are in \textit{canonical form}. Conversions between these forms are given in~\cite{robbins2024mixed}.
Sets in 01-form use the shorthand notation $\langle\cdot\rangle_{01}$. We similarly use this notation to indicate zonotopes or constrained zonotopes defined such that $\bm{\xi} \in [0,1]^{n_G}$.

One major advantage of zonotopic sets when compared to other set representations is that they have closed form expressions for key set operations. In particular, for hybrid zonotopes,  
\begin{subequations} \label{eq:hz-set-ops}
\allowdisplaybreaks
\begin{align}
&R \mathcal{Z} + \mathbf{s} = \left\langle RG_c, R G_b, R\mathbf{c} + \mathbf{s}, A_c, A_b, \mathbf{b} \right\rangle \;, \label{eq:hz-set-ops-lin-map} \\
&\mathcal{Z}_1 \times \mathcal{Z}_2 = \left\langle \begin{bmatrix}
    G_{c1} & 0 \\
    0 & G_{c2}
\end{bmatrix}, \begin{bmatrix}
    G_{b1} & 0 \\
    0 & G_{b2}
\end{bmatrix}, \begin{bmatrix}
    \mathbf{c}_1 \\ \mathbf{c}_2
\end{bmatrix}, \right. \nonumber \\
&\qquad \qquad \qquad \quad \left. \begin{bmatrix}
    A_{c1} & 0 \\
    0 & A_{c2}
\end{bmatrix}, \begin{bmatrix}
    A_{b1} & 0 \\
    0 & A_{b2}
\end{bmatrix}
\begin{bmatrix}
    \mathbf{b}_1 \\ \mathbf{b}_2
\end{bmatrix} \right\rangle \;, \label{eq:hz-set-ops-cart-prod} \\
&\mathcal{Z}_1 \oplus \mathcal{Z}_2 = \left\langle \begin{bmatrix}
    G_{c1} & G_{c2}
\end{bmatrix}, \begin{bmatrix}
    G_{b1} & G_{b2}
\end{bmatrix}, \mathbf{c}_1 + \mathbf{c}_2, \right. \nonumber \\
&\qquad \qquad \qquad \quad \left.
\begin{bmatrix}
    A_{c1} & 0 \\
    0 & A_{c2}
\end{bmatrix}, 
\begin{bmatrix}
    A_{b1} & 0 \\
    0 & A_{b2}
\end{bmatrix}
\begin{bmatrix}
    \mathbf{b}_1 \\ \mathbf{b}_2
\end{bmatrix} \right\rangle \;, \label{eq:hz-set-ops-mink-sum} \\
&\mathcal{Z}_1 \cap_R \mathcal{Z}_2 = \left\langle \begin{bmatrix}
    G_{c1} & 0
\end{bmatrix}, 
\begin{bmatrix}
    G_{b1} & 0
\end{bmatrix},
\mathbf{c}_1, \right. \nonumber \\
&\qquad \left. \begin{bmatrix}
    A_{c1} & 0 \\
    0 & A_{c2} \\
    R G_{c1} & -G_{c2}
\end{bmatrix}, 
\begin{bmatrix}
    A_{b1} & 0 \\
    0 & A_{b2} \\
    R G_{b1} & -G_{b2}
\end{bmatrix}, 
\begin{bmatrix}
    \mathbf{b}_1 \\ \mathbf{b}_2 \\ \mathbf{c}_2 - R \mathbf{c}_1
\end{bmatrix} \right\rangle \;, \label{eq:hz-set-ops-intersection} \\
&\mathcal{Z}_1 \cap_R \mathcal{H} = \left\langle \begin{bmatrix} G_{c1} & 0 \end{bmatrix}, G_{b1}, \mathbf{c}_1, \right. \nonumber \\
    &\qquad \left. \begin{bmatrix} 
        A_{c1} & \mathbf{0} \\
        \mathbf{h}^T R G_{c1} & \frac{d_m}{2}
    \end{bmatrix}, \begin{bmatrix}
        A_{b1} \\
        \mathbf{h}^T R G_{b1}
    \end{bmatrix}, \begin{bmatrix}
        \mathbf{b}_1 \\
        f - \mathbf{h}^T R \mathbf{c} - \frac{d_m}{2}
    \end{bmatrix} \right\rangle \;, \label{eq:hz-set-ops-halfspace-intersection}
\end{align}
\end{subequations}
where \eqref{eq:hz-set-ops-lin-map} is the affine map, \eqref{eq:hz-set-ops-cart-prod} is the Cartesian product, \eqref{eq:hz-set-ops-mink-sum} is the Minkowski sum, and \eqref{eq:hz-set-ops-intersection} is the generalized intersection. \revmargin{2}{R2.3}\revres{2}{Eq.~\eqref{eq:hz-set-ops-halfspace-intersection} is the halfspace intersection, where $\mathcal{H} = \left\{ \mathbf{z} ~\middle|~ \mathbf{h}^T \mathbf{z} \leq f \right\}$ and $d_m = f - \mathbf{h}^T R \mathbf{c}_1 + \sum_{i=1}^{n_{Gc,1}} |\mathbf{h}^T R \mathbf{g}_{c,i}| + \sum_{i=1}^{n_{Gb,1}} |\mathbf{h}^T R \mathbf{g}_{b,i}|$.}

\subsection{Convex Relaxations}
The convex relaxation of a mixed-integer set representation is the set obtained by replacing any integrality constraints with their interval hulls. For a hybrid zonotope $\mathcal{Z} = \langle G_c, G_b, \mathbf{c}, A_c, A_b, \mathbf{b} \rangle$, the convex relaxation is the constrained zonotope
\begin{equation}
    \mathit{CR}(\mathcal{Z}) = \left\langle \begin{bmatrix} G_c & G_b \end{bmatrix}, \mathbf{c}, \begin{bmatrix} A_c & A_b \end{bmatrix}, \mathbf{b} \right\rangle \;.
\end{equation}

By definition, $\mathit{CR}(\mathcal{Z}) \supseteq \mathit{CH}(\mathcal{Z}) \supseteq \mathcal{Z}$, where $\mathit{CH}(\cdot)$ denotes the convex hull. Sets $\mathcal{Z}$ for which $\mathit{CR}(\mathcal{Z}) = \mathit{CH}(\mathcal{Z})$ are termed \textit{sharp}. Operations~\eqref{eq:hz-set-ops-lin-map}-\eqref{eq:hz-set-ops-mink-sum} preserve sharpness while~\eqref{eq:hz-set-ops-intersection} does not~\cite{glunt2025sharp}. Sharpness, and by extension tightness (i.e., how much larger $\mathit{CR}(\mathcal{Z})$ is than $\mathcal{Z}$), is a key property in mixed-integer optimization~\cite{kronqvist2025psplit}.

\subsection{Alternating Direction Method of Multipliers} \label{sec:admm-prelims}
ADMM solves optimization problems of the form
\begin{subequations} \label{eq:admm-general}
\begin{align}
    &\min_{\mathbf{x}, \mathbf{z}} f(\mathbf{x}) + g(\mathbf{z})\;, \\
    &\text{s.t.}\; A\mathbf{x} + B \mathbf{z} = \mathbf{c} \;, \label{eq:admm-general-eq-cons}
\end{align}
\end{subequations}
where $f : \real^n \rightarrow \real \cup \{ \infty \}$ and $g : \real^m \rightarrow \real \cup \{ \infty \}$~\cite{boyd2011distributed}.

The algorithm is defined by the iterations
\begin{subequations} \label{eq:admm-iterations-general}
\begin{align}
    &\mathbf{x}_{k+1} = \argmin_{\mathbf{x}} \left(f(\mathbf{x}) + \frac{\rho}{2} ||A \mathbf{x} + B \mathbf{z}_k - \mathbf{c} + \mathbf{u}_k||_2^2 \right) \;, \label{eq:admm-gen-a} \\
    &\mathbf{z}_{k+1} = \argmin_{\mathbf{z}} \left(g(\mathbf{z}) + \frac{\rho}{2} ||A \mathbf{x}_{k+1} + B \mathbf{z} - \mathbf{c} + \mathbf{u}_k||_2^2 \right) \;, \label{eq:admm-gen-b} \\
    &\mathbf{u}_{k+1} = \mathbf{u}_k + A \mathbf{x}_{k+1} + B \mathbf{z}_{k+1} - \mathbf{c} \;, \label{eq:admm-gen-c}
\end{align}
\end{subequations}
where $\rho>0$ is the ADMM penalty parameter which weighs satisfaction of the constraint \eqref{eq:admm-general-eq-cons}. 

The convergence of~\eqref{eq:admm-iterations-general} to feasible and optimal $\mathbf{x}^*$, $\mathbf{z}^*$, and $\mathbf{u}^*$ is guaranteed when 1) $f$ and $g$ are closed, proper, and convex, and 2) the augmented Lagrangian 
\begin{equation}
    \mathcal{L} = f(\mathbf{x}) + g(\mathbf{z}) + \frac{\rho}{2} ||A \mathbf{x} + B \mathbf{z} - \mathbf{c} + \mathbf{u}||_2^2 \;,
\end{equation}
has a saddle point~\cite{boyd2011distributed}.

\subsection{The Feasibility Pump} \label{sec:fp}
The feasibility pump as originally proposed in~\cite{fischetti2005feasibility} is a heuristic algorithm to find feasible solutions to mixed-integer linear programs (MILPs), i.e., 
\begin{subequations} \label{eq:milp}
\begin{align}
    &\min_{\mathbf{z}} \mathbf{q}^T \mathbf{z} \;, \\
    &\text{s.t.}\; C \mathbf{z} \leq \mathbf{d},\; z_j \in \{0,1\} \; \forall j \in \mathcal{I} \;,
\end{align}
\end{subequations}
where $\mathcal{I}$ is the set of integer variable indices. For clarity of exposition, we assume that all integer variables are binary.

The basic version of the feasibility pump is given in Algorithm~\ref{alg:fp}. For convenience, we define the sets $\mathcal{C}$ and $\mathcal{R}$ to be
\begin{subequations} \label{eq:fp-sets}
\begin{align}
    &\mathcal{C} = \{\mathbf{z} \revmargin{3}{R3.2}~|~ C \mathbf{z} \leq \mathbf{d} \} \;, \\
    &\mathcal{R} = \left\{ \mathbf{z} \revmargin{3}{R3.2}~\middle|~ z_j \in \begin{cases}
        \{0, 1\}, &j\in \mathcal{I} \\
       (-\infty, \infty),& \text{otherwise} 
    \end{cases} \right\} \;.
\end{align}
\end{subequations}

The feasibility pump works by driving two sequences of iterates, $\mathbf{z}^* \in \mathcal{C}$ and $\tilde{\mathbf{z}} \in \mathcal{R}$, towards each other. To do this, the feasibility pump repeatedly solves the LP
\begin{subequations} \label{eq:fp_delta}
\begin{align}
    &\min_{\mathbf{z}} \mathbf{s}^T \mathbf{z} \;, \label{eq:fp_delta_obj} \\
    &\text{s.t.}\; C \mathbf{z} \leq \mathbf{d} \;,
\end{align}
\end{subequations}
where the elements of $\mathbf{s}$ are defined as
\begin{equation}
    s_j = \begin{cases}
        \hphantom{-}1, & j \in \mathcal{I},\; \tilde{z}_j=1 \;, \\
        -1, & j \in \mathcal{I},\; \tilde{z}_j=0 \;, \\
        \hphantom{-}0, & j \notin \mathcal{I} \;.
    \end{cases}
\end{equation}
Eq.~\eqref{eq:fp_delta} returns the nearest point $\mathbf{z}$ inside the polyhedron $\mathcal{C}$ to the binary elements of the binary feasible point $\tilde{\mathbf{z}}$.

Because the feasibility pump can cycle, i.e., revisit the same points $\tilde{\mathbf{z}}$, two different random perturbation schemes are implemented. For cycles of length one, the elements of $\tilde{z}_j,\; j \in \mathcal{I}$ for which $z^*_j$ is most fractional are flipped. The number of elements to flip is controlled by the parameter $t$. For longer cycles, a restart procedure is implemented where binary variables $\tilde{z}_j,\; j \in \mathcal{I}$ are flipped with probability dependent on the fractionality of $z^*_j$.

\begin{algorithm}
    \caption{The Feasibility Pump~\cite{fischetti2005feasibility}\\
    Inputs: $\mathbf{q}$, $\mathcal{C}$, $\mathcal{R}$, $t$ \\
    Ouputs: $\mathbf{z}^* \in \mathcal{C} \cap \mathcal{R}$}
    \begin{algorithmic}[1]
        \State $\mathbf{z}^* \gets$ solution to convex relaxation of~\eqref{eq:milp}
        \State \textbf{if} $\mathbf{z}^* \in \mathcal{R}$, \Return $\mathbf{z}^*$
        \State $\tilde{\mathbf{z}} \gets \pi_{\mathcal{R}}(\mathbf{z}^*)$
        \While{True}
            \State $\mathbf{z}^* \gets$ solution to~\eqref{eq:fp_delta}
            \State \textbf{if} $\mathbf{z}^* \in \mathcal{R}$, \Return $\mathbf{z}^*$
            \If{cycle detected}
                \LineComment{restart procedure} \label{alg:line:fp-restart}
                \For{$j \in \mathcal{I}$}
                    \State $\rho_j \gets \mathrm{rand}(-0.3, 0.7)$
                    \If{$|z_j^* - \tilde{z}_j| + \max\{\rho_j, 0 \} > 0.5$}
                        \State $\tilde{z}_j \gets 0$ if $\tilde{z}_j=1$, $\tilde{z}_j \gets 1$ if $\tilde{z}_j = 0$
                    \EndIf
                \EndFor
            \Else
                \If{$\exists j \in \mathcal{I}$ s.t. $(\pi_{\mathcal{R}}(\mathbf{z}^*))_j \neq \tilde{z}_j$}
                    \State $\tilde{\mathbf{z}} \gets \pi_{\mathcal{R}}(\mathbf{z}^*)$
                \Else
                    \LineComment{perturbation procedure} \label{alg:line:fp-perturb}
                    \State $t_r \gets \mathrm{rand}(t/2, 3t/2)$
                    \State $\mathcal{I}_F \gets t_r$ indices $j \in \mathcal{I}$ with highest $|z^*_j - \tilde{z}_j|$ 
                    \State $\forall j \in \mathcal{I}_F:$ $\tilde{z}_j \gets 0$ if $\tilde{z}_j=1$, $\tilde{z}_j \gets 1$ if $\tilde{z}_j = 0$
                \EndIf
            \EndIf
        \EndWhile
    \end{algorithmic}
    \label{alg:fp}
\end{algorithm}

The objective feasibility pump (OFP)~\cite{achterberg2007improving} is a popular modification to the feasibility pump that uses the original problem objective during its search for a feasible solution in order to improve the optimality of found solutions. Specifically, Eq.~\eqref{eq:fp_delta_obj} is replaced with
\begin{equation} \label{eq:ofp_delta}
    \min_{\mathbf{z}} \left((1-\alpha) \mathbf{s} + \alpha \frac{||\mathbf{s}||_2}{||\mathbf{q}||_2} \mathbf{q} \right)^T \mathbf{z} \;,
\end{equation}
where the parameter $\alpha$ is a weighting factor that is initialized to $\alpha_0$ and decremented by a constant factor $\phi$ at every iteration. Thus, the OFP prioritizes the original problem objective during early iterations and feasibility during later iterations. Cycle detection in the OFP requires that, at iterations $k$ and $l$, $\tilde{\mathbf{z}}_k = \tilde{\mathbf{z}}_l$ within numerical tolerance, and also that $|\alpha_k - \alpha_l| \leq \delta_{\alpha}$ where $\delta_{\alpha}$ is a user-specified parameter.

\section{Reachability of Piecewise Affine Systems with Hybrid Zonotopes} \label{sec:hz_reach}

This section shows how reachability analysis with hybrid zonotopes can be used to construct efficient planning problem formulations for hybrid systems \revmargin{2}{R2.10}\revres{2}{online}.
Reachability calculations for PWA systems are presented and shown to achieve reduced memory complexity and tighter convex relaxations as compared to state-of-the-art methods for MLD systems. Extensions to lifted problem formulations with state constraints are provided, enabling the formulation of constrained planning problems for hybrid systems. 

\subsection{Problem Statement}

This section addresses optimal planning for PWA systems given by
\begin{subequations} \label{eq:piecewise-affine-system}
\begin{align}
    &\mathbf{x}_{k+1} = A^i_k \mathbf{x}_k + B^i_k \mathbf{u}_k + \mathbf{f}^i_k \;, \label{eq:pwa-system-dyn} \\
    &\begin{bmatrix}
        \mathbf{x}_k \\ \mathbf{u}_k
    \end{bmatrix} \in \revmargin{2}{R2.7}\revres{2}{\mathcal{D}}^i_k \;, \label{eq:pwa-system-mode}
\end{align}
\end{subequations}
where $i \in \{1, ..., p\}$ is the index of a given dynamic mode, and $\revmargin{2}{R2.7}\revres{2}{\mathcal{D}}^i_k \subseteq \real^{n_x+n_u}$ is the corresponding domain~\cite{borrelli2017predictive}. In words, the active mode of the PWA dynamics~\eqref{eq:pwa-system-dyn} depends on the domain~\eqref{eq:pwa-system-mode} to which the states and inputs of the system belong.
\begin{assumption} \label{ass:union_of_SU}
    The state and input domain for system~\eqref{eq:piecewise-affine-system} over all modes is separable into state and input components, i.e.,
    $\bigcup_{i \in \{1, ..., p\}} \revmargin{2}{R2.7}\revres{2}{\mathcal{D}}^i_k = \mathcal{S}_k \times \mathcal{U}_k$.
\end{assumption}

Defining planning horizon $N$, quadratic cost matrices $Q, R, Q_N \geq 0$, reference trajectory $\mathbf{x}^r_k$, initial condition set $\mathcal{X}_0$, and time-varying state constraint set $\mathcal{F}_k \subseteq \mathcal{S}_k$, the optimal planning problem is stated as
\begin{subequations} \label{eq:planning-problem-nonlifted}
\begin{align}
    &\revmargin{3}{R3.5}\revres{3}{\min_{\{\mathbf{x}_0, ..., \mathbf{x}_{N}\}, \{\mathbf{u}_0, ... \mathbf{u}_{N-1}\}}} \frac{1}{2} \sum_{k=0}^{N-1} \left((\mathbf{x}_k-\mathbf{x}^r_k)^T Q (\mathbf{x}_k-\mathbf{x}^r_k) \right. \nonumber \\
    &\qquad \left. + \mathbf{u}_k^T R \mathbf{u}_k \right) + \frac{1}{2}(\mathbf{x}_N-\mathbf{x}^r_N)^T Q_N (\mathbf{x}_N-\mathbf{x}^r_N)\;, \\
    &\text{s.t.}~\forall k \in \{0,...,N-1\}:\mathbf{x}_0 \in \mathcal{X}_0,\; \mathbf{x}_{k+1} \in \mathcal{F}_{k+1}\;, \\
    &\hphantom{\text{s.t.}}~\mathbf{x}_k, \mathbf{u}_k, \mathbf{x}_{k+1}~\text{satisfy~\eqref{eq:piecewise-affine-system}} \;.
\end{align}
\end{subequations}
Eq.~\eqref{eq:planning-problem-nonlifted} requires planned trajectories to respect the PWA system dynamics while adhering to state constraints.

\begin{assumption}
    The initial condition set belongs to the state domain, i.e., $\mathcal{X}_0 \subseteq \mathcal{S}_0$.
\end{assumption}
Typically, the initial condition is a point such that $\mathcal{X}_0 = \{\mathbf{x}_0\}$.

\begin{assumption}
    The domain sets $\mathcal{S}_k$, $\mathcal{U}_k$, $\revmargin{2}{R2.7}\revres{2}{\mathcal{D}}_k^i$, constraint sets $\mathcal{F}_{k+1}$, and initial condition set $\mathcal{X}_0$ are polytopes or unions of polytopes $\forall k \in \{0, ..., N-1\}$ and $\forall i \in \{1,...,p\}$.
\end{assumption}

\subsection{Lifted Planning Problem}

We will seek to construct a lifted set representing the feasible space of the planning problem~\eqref{eq:planning-problem-nonlifted}. First, the following definitions are necessary:
\begin{defn} \label{def:successor-set}
    For a dynamic system $\mathbf{x}_{k+1} = \mathbf{f}(\mathbf{x}_k, \mathbf{u}_k)$ with $\mathbf{x}_k \in \mathcal{X}_k$ and $\mathbf{u}_k \in \mathcal{U}_k$, the successor set of $\mathcal{X}_k$ is
    \begin{equation}
        \mathrm{Suc}(\mathcal{X}_k, \mathcal{U}_k) = \left\{ \mathbf{f}(\mathbf{x}_k, \mathbf{u}_k) \revmargin{3}{R3.2}~\middle|~ \mathbf{x}_k \in \mathcal{X}_k,\; \mathbf{u}_k \in \mathcal{U}_k \right\} \;.
    \end{equation}
\end{defn}
\begin{defn} \label{def:reachable-set}
    Given an initial set $\mathcal{X}_0$ and input sets $\mathcal{U}_k,\; \forall k \in \{1, ..., N-1\}$, the $N$-step forward reachable set is defined by the recursion
    \begin{equation}
        \mathcal{X}_k = \mathrm{Suc}(\mathcal{X}_{k-1}, \mathcal{U}_{k-1}),\; k \in \{1, ..., N\} \;.
    \end{equation}
\end{defn}

The lifted set $\mathcal{Z}_N$, defined as follows, represents the feasible space of~\eqref{eq:planning-problem-nonlifted}.
\begin{defn} \label{def:Z_def}
    The set $\mathcal{Z}_N$ is defined such that any point $\mathbf{z} \in \mathcal{Z}_N$ satisfies
    \begin{subequations} 
    \begin{align}
        &\mathbf{z} = \begin{bmatrix}
             \mathbf{x}_0^T & \mathbf{u}_0^T & \cdots & \mathbf{u}_{N-1}^T & \mathbf{x}_N^T
        \end{bmatrix}^T \;, \label{eq:Z_def_a} \\
        &\mathbf{x}_0 \in \mathcal{X}_0 \;, \label{eq:Z_def_b} \\
        &\forall k \in \{0,..,N-1\}: \nonumber \\
        &\hphantom{\forall k} \mathbf{x}_k \in \mathcal{S}_k,\; \mathbf{u}_k \in \mathcal{U}_k,\; \mathbf{x}_{k+1} \in \mathrm{Suc}(\{\mathbf{x}_k\}, \{\mathbf{u}_k\}) \cap \mathcal{F}_{k+1} \;. \label{eq:Z_def_c}
    \end{align}
    \end{subequations}
\end{defn}

Neglecting constant terms, the lifted cost function and associated optimization problem are given as
\begin{subequations} \label{eq:optim_prob_gen}
\begin{align}
    &\min_{\mathbf{z}} \frac{1}{2} \mathbf{z}^T P \mathbf{z} + \mathbf{q}^T \mathbf{z} \;, \\
    &\mathrm{s.t.}\; \mathbf{z} \in \mathcal{Z}_N \;,
\end{align}
\end{subequations}
where 
\begin{subequations}
\begin{align}
    &P = \mathrm{blkdiag}([Q, R, Q, ..., Q_N]) \;, \\
    &\mathbf{q} = \begin{bmatrix} \mathbf{0}^T & \mathbf{0}^T & -(Q \mathbf{x}^r_1)^T & \cdots & -(Q_N \mathbf{x}_N^r)^T \end{bmatrix}^T \;.
\end{align}
\end{subequations}

\subsection{Graphs of Functions}
We leverage graphs of functions
to construct $\mathcal{Z}_N$ for~\eqref{eq:optim_prob_gen} via reachability analysis with hybrid zonotopes.
Hybrid zonotopes were used to compute reachable sets for MLD systems in~\cite{bird2023hybrid}, and were used with graphs of functions to over-approximate reachable sets of nonlinear systems in~\cite{siefert2025reachability}. Outside of a hybrid zonotope context, graphs of functions were used for PWA system control in~\cite{marcucci2019mixed}.

\begin{defn} \label{def:graph-of-func}
The graph of the function $\psi(\cdot)\revmargin[true]{3}{R3.6}\revres{3}{~:~\mathcal{D} \rightarrow \real^n}$ is the set 
\begin{equation}
    \Psi = \left\{ \begin{bmatrix}
        \mathbf{p}^T & \mathbf{q}^T
    \end{bmatrix}^T~\middle|~ \mathbf{p} \in \mathcal{D}, \mathbf{q} \revmargin{3}{R3.7}\revres{3}{=} \psi(\mathbf{p}) \right\} \;,
\end{equation}
where $\mathcal{D}$ is the domain set of the inputs $\mathbf{p}$.
\end{defn}

Assuming input domain $\begin{bmatrix} \mathbf{x}_k^T & \mathbf{u}_k^T \end{bmatrix}^T \in \mathcal{S}_k \times \mathcal{U}_k$,
the graph of the function $\mathbf{f}(\mathbf{x}_k, \mathbf{u}_k)$ is then
\begin{equation} \label{eq:gof-def-general-system}
    \Psi_k = \left\{ \begin{bmatrix}
        \mathbf{x}_k \\ \mathbf{u}_k \\ \mathbf{x}_{k+1}
    \end{bmatrix} \revmargin{3}{R3.2}~\middle|~ 
    \begin{aligned}
        &\begin{bmatrix}
        \mathbf{x}_k \\ \mathbf{u}_k \end{bmatrix} \in \revmargin{2}{R2.7}\revres{2}{\mathcal{D}}_k \;, \\
        &\mathbf{x}_{k+1} \in \mathrm{Suc}(\{\mathbf{x}_k\}, \{\mathbf{u}_k\})
    \end{aligned} \right\} \;.
\end{equation}

For system~\eqref{eq:piecewise-affine-system}, the graph of the function at timestep $k$ is
\begin{equation} \label{eq:pwa_graph_of_function}
    \Psi_k = \left\{ \begin{bmatrix}
        \mathbf{x}_k \\ \mathbf{u}_k \\ \mathbf{x}_{k+1}
    \end{bmatrix} \revmargin{3}{R3.2}~\middle|~ 
    \begin{aligned}
        &\begin{bmatrix}
        \mathbf{x}_k \\ \mathbf{u}_k \end{bmatrix} \in \mathcal{S}_k \times \mathcal{U}_k \\
        &\mathbf{x}_{k+1} = \begin{aligned}
            &A^i_k \mathbf{x}_k + B^i_k \mathbf{u}_k + \mathbf{f}^i_k \;, \\
            &\hphantom{A^i_k \mathbf{x}_k}[\mathbf{x}^T \; \mathbf{u}_k^T]^T \in \revmargin{2}{R2.7}\revres{2}{\mathcal{D}}^i_k
        \end{aligned}
    \end{aligned} 
    \right\} \;.
\end{equation}

To construct $\Psi_k$, we first compute the graph of the function $\Psi_k^i$ for each mode $i$ in the system as
\begin{subequations} \label{eq:pwa_single_mode_GOF}
\begin{align} 
    \Psi^i_k &= \left\{ \begin{bmatrix}
        \mathbf{x}_k \\ \mathbf{u}_k \\ \mathbf{x}_{k+1}
    \end{bmatrix} \revmargin{3}{R3.2}~\middle|~ 
    \begin{aligned}
        &\begin{bmatrix}
        \mathbf{x}_k \\ \mathbf{u}_k \end{bmatrix} \in \revmargin{2}{R2.7}\revres{2}{\mathcal{D}}^i_k \\
        &\mathbf{x}_{k+1} = A^i_k \mathbf{x}_k + B^i_k \mathbf{u}_k + \mathbf{f}^i_k
    \end{aligned} 
    \right\} \;, \\
     &= \begin{bmatrix}
        I & 0 \\
        0 & I \\
        A^i_k & B^i_k
    \end{bmatrix} \revmargin{2}{R2.7}\revres{2}{\mathcal{D}}^i_k \oplus 
    \left\{\begin{bmatrix} \mathbf{0} \\ \mathbf{0} \\ \mathbf{f}^i_k \end{bmatrix}\right\} \;. \label{eq:pwa_single_mode_GOF_implementation}
\end{align}
\end{subequations}

Eq.~\eqref{eq:pwa_single_mode_GOF_implementation} is readily implemented using hybrid zonotopes via~\eqref{eq:hz-set-ops-lin-map} and \eqref{eq:hz-set-ops-mink-sum}. 
\revmargin{2}{R2.8}\revres{2}{The graph of the function for~\eqref{eq:piecewise-affine-system} is constructed by taking the union of the graphs of functions for each mode, i.e.,
\begin{equation} \label{eq:union_of_gofs}
    \Psi_k = \bigcup_{i \in \{1, ..., p\}} \Psi^i_k \;.
\end{equation}}

\begin{remark}
Implementing \eqref{eq:union_of_gofs} with hybrid zonotopes requires a hybrid zonotope union operation. The appendix provides three different union identities: 1) A sharp union, 2) a condensed union, and 3) a zonotope union. The sharp union ensures that, if the $\Psi_k^i$ are sharp, then $\Psi_k$ is sharp. \revmargin{2}{R2.9}\revres{2}{Note that constrained zonotopes are trivially sharp, as they are hybrid zonotopes with no binary variables.} The condensed union is a lower memory complexity union operation that does not preserve sharpness in general. The zonotope union applies specifically to the case that all $\Psi^i_k$ are zonotopes and is most efficient when the $\Psi_k^i$ zonotopes share common generators. The zonotope union produces a sharp $\Psi_k$. 
See~\cite{glunt2025sharp} for more information about sharp hybrid zonotopes.
\end{remark}

\subsection{Reachability Analysis Using Graphs of Functions}
When performing reachability analysis with graphs of functions, it is useful to define state and input bounds $\overline{\mathcal{S}}$ and $\overline{\mathcal{U}}$ such that $\forall k:\; \overline{\mathcal{S}} \supseteq \mathcal{S}_k,\; \overline{\mathcal{U}} \supseteq \mathcal{U}_k$. These sets are generally chosen to have low representational complexity; in Secs.~\ref{sec:numerical-results} and \ref{sec:experiment}, we choose $\overline{\mathcal{S}}$ and $\overline{\mathcal{U}}$ to be boxes represented as zonotopes. 

Given an initial set $\mathcal{X}_0$ for the state of system~\eqref{eq:piecewise-affine-system}, the reachable sets at time step $k$ are given by the recursion~\cite{siefert2025reachability}
\begin{equation} \label{eq:gof_recursion_orig}
    \mathcal{X}_{k+1} = \begin{bmatrix} 0 & 0 & I \end{bmatrix} \left( \Psi_k \cap_{\begin{bmatrix} I & 0 & 0 \\ 0 & I & 0 \end{bmatrix}} (\mathcal{X}_k \times \overline{\mathcal{U}}) \right) \;.
\end{equation}

Eq.~\eqref{eq:gof_recursion_orig} works well when we are only interested in $\mathcal{X}_N$, but cannot readily be extended to construction of the lifted set $\mathcal{Z}_N$ defined in Def.~\ref{def:Z_def}. To address this challenge, Lemma~\ref{lemma:gof_recursion_lifted} and Thm.~\ref{thm:gof_recursion_lifted} give a modified recursion for lifted reachability analysis using graphs of functions. 

To enforce state constraints, we define the state-constrained graph of function
\begin{equation} \label{eq:gof-constrained}
    \tilde{\Psi}_k = \Psi_k \cap_{\begin{bmatrix} 0 & 0 & I \end{bmatrix}} \mathcal{F}_{k+1} \;.
\end{equation}

\begin{lemma} \label{lemma:gof_recursion_lifted}
    Given $\mathcal{Z}_k$ that satisfies Def.~\ref{def:Z_def} and state-constrained graph of function $\tilde{\Psi}_k$ at time step $k$, the set 
    \begin{equation} \label{eq:lifted-gof-recursion}
        \mathcal{Z}_{k+1} = (\mathcal{Z}_k \times \overline{\mathcal{U}} \times \overline{\mathcal{S}}) \cap_{\begin{bmatrix}
                0 & \dots & 0 & I_{n_x+n_u+n_x}
            \end{bmatrix}} \tilde{\Psi}_k \;,
    \end{equation}  
    satisfies Def.~\ref{def:Z_def}.
    \begin{proof}
        Define the set $\mathcal{Z}'_k = \mathcal{Z}_k \times \overline{\mathcal{U}} \times \overline{\mathcal{S}}$. For any vector $\mathbf{z}'_k \in \mathcal{Z}'_k$, the segments $\mathbf{x}_k'$, $\mathbf{u}_k'$, and $\mathbf{x}'_{k+1}$ defined as in Def.~\ref{def:Z_def} satisfy $\mathbf{x}'_k \in \mathcal{X}_k$, $\mathbf{u}'_k \in \overline{\mathcal{U}}$, and $\mathbf{x}'_{k+1} \in \overline{\mathcal{S}}$ by construction. From~\eqref{eq:gof-def-general-system}, the corresponding elements of $\mathcal{Z}_{k+1}$ satisfy the inclusions $\mathbf{x}_k \in \mathcal{X}_k \cap \mathcal{S}_k \subseteq \mathcal{S}_k$, $\mathbf{u}_k \in \overline{\mathcal{U}} \cap \mathcal{U}_k = \mathcal{U}_k$, and $\mathbf{x}_{k+1} \in (\overline{\mathcal{S}} \cap \mathcal{S}_{k+1}) \cap \mathrm{Suc}(\{\mathbf{x}_k \}, \{\mathbf{u}_k \}) = \mathcal{S}_{k+1} \cap \mathrm{Suc}(\{\mathbf{x}_k \}, \{\mathbf{u}_k \})$. Since generalized intersection cannot expand the set, \eqref{eq:Z_def_b} holds, \eqref{eq:Z_def_c} holds for all $k' < k$, and $\mathcal{Z}_{k+1}$ satisfies Def.~\ref{def:Z_def}.
    \end{proof}
\end{lemma}

\begin{theorem} \label{thm:gof_recursion_lifted}
    Given a discrete time system with state-constrained graphs of functions $\tilde{\Psi}_k$ and initial state set $\mathcal{X}_0$,
    the recursion given in~\eqref{eq:lifted-gof-recursion} 
    with $\mathcal{Z}_0 = \mathcal{X}_0$ produces a lifted reachable set $\mathcal{Z}_N$ satisfying Def.~\ref{def:Z_def}.
    \begin{proof}
        Using the fact that $\mathcal{Z}_0 = \mathcal{X}_0$ satisfies Def.~\ref{def:Z_def} for $N=0$, applying Lemma~\ref{lemma:gof_recursion_lifted} recursively gives $\mathcal{Z}_N$ satisfying Def.~\ref{def:Z_def}.
    \end{proof}
\end{theorem}

Using Thm.~\ref{thm:gof_recursion_lifted}, Algorithm~\ref{alg:online_planning_problem_formulation} gives a methodology 
\revmargin[true]{2}{R2.10}\revres{2}{that enables efficient online formulation of planning problems for PWA systems.}
The graphs of functions $\Psi_k$ used in Algorithm~\ref{alg:online_planning_problem_formulation} are computed online using \eqref{eq:union_of_gofs}. 

\begin{algorithm}
    \caption{Build hybrid system planning problem~\eqref{eq:optim_prob_gen} using hybrid zonotope reachability analysis \\
    Inputs: $\mathcal{X}_0$, $\tilde{\Psi}_k$, $\mathbf{x}^r_k$, $Q$, $R$, $Q_N$ \\
    Outputs: $\mathcal{Z}_N$, $P$, $\mathbf{q}$}
    \begin{algorithmic}[1]
        \State $\mathcal{Z}_0 \gets \mathcal{X}_0$
        \State $P \gets Q$
        \State $\mathbf{q} \gets \mathbf{0}$ 
        \For{$k \in \{1, ..., N\}$}
            \State $\mathcal{Z}_{k+1} \gets (\mathcal{Z}_k \times \overline{\mathcal{U}} \times \overline{\mathcal{S}}) \cap_{\begin{bmatrix}
                0 & \dots & 0 & I
            \end{bmatrix}} \tilde{\Psi}_k$
            \If{$k = N$}
                \State $P \gets \mathrm{blkdiag}\left( \begin{bmatrix} P & R & Q_N \end{bmatrix} \right)$
                \State \State $\mathbf{q} \gets \begin{bmatrix} \mathbf{q}^T & \mathbf{0}^T & -(Q_N \mathbf{x}^r_k)^T \end{bmatrix}^T$
            \Else{}
                \State $P \gets \mathrm{blkdiag}\left( \begin{bmatrix} P & R & Q  \end{bmatrix} \right)$
                \State \State $\mathbf{q} \gets \begin{bmatrix} \mathbf{q}^T & \mathbf{0}^T & -(Q\mathbf{x}^r_k)^T \end{bmatrix}^T$
            \EndIf            
        \EndFor
        \State \Return $\left( \mathcal{Z}_N, P, \mathbf{q} \right)$
    \end{algorithmic}
    \label{alg:online_planning_problem_formulation}
\end{algorithm}

\subsection{Convex Relaxation Tightness of Reachable Sets}
\revmargin{2}{R2.3}\revres{2}{Next, we show that the proposed PWA system reachability calculations result in hybrid zonotopes with tighter convex relaxations than existing reachability calculations for hybrid systems. For simplicity of exposition, we assume in this subsection that $\mathcal{U}$ and $\Psi$ are time-invariant; however, the results are readily extended to cases where they are time-varying.}

\revres{2}{Assume a reachable set recursion of the form
\begin{equation} \label{eq:reachable-set-recursion-form}
    \mathcal{X}_{k+1} = \left\{ \mathbf{x}_{k+1}~\middle|~\mathbf{x}_k \in \mathcal{X}_k,\; \mathbf{u}_k \in \mathcal{U},\; \begin{bmatrix} \mathbf{x}_k \\ \mathbf{u}_k \\ \mathbf{x}_{k+1} \end{bmatrix} \in \Psi \right\} \;.
\end{equation}
}

\revres{2}{\begin{lemma} \label{lemma:containment-preserved}
    For recursions $\mathcal{X}_{k+1}^A$ and $\mathcal{X}_{k+1}^B$ satisfying~\eqref{eq:reachable-set-recursion-form}, if $\mathcal{X}_k^A \subseteq \mathcal{X}_k^B$ and $\Psi^A \subseteq \Psi^B$, then $\mathcal{X}_{k+1}^A \subseteq \mathcal{X}_{k+1}^B$.
    \begin{proof}
        $\mathbf{x}_k \in \mathcal{X}_k^B~\forall\mathbf{x}_k \in \mathcal{X}_k^A$, and $[ \mathbf{x}_k^T~\mathbf{u}_k^T~\mathbf{x}_{k+1}^T]^T \in \Psi^B~\forall [ \mathbf{x}_k^T~\mathbf{u}_k^T~\mathbf{x}_{k+1}^T]^T \in \Psi^A$.
    \end{proof}
\end{lemma}}

\revres{2}{\begin{lemma} \label{lemma:convex-relaxation-commutes}
    The convex relaxation commutes with all identities in~\eqref{eq:hz-set-ops}.
    \begin{proof}
        From~\cite[Lemma 1]{glunt2025sharp}, $\mathrm{CR}(\mathcal{Z}_1 \cap \mathcal{Z}_2) = \mathrm{CR}(\mathcal{Z}_1) \cap \mathrm{CR}(\mathcal{Z}_2)$ by inspection of the intersection identity. Proofs for all other identities follow in kind.
    \end{proof}
\end{lemma}}

\revres{2}{\begin{theorem} \label{thm:pwa-reachability-tightness}
    Given reachable set recursions $\mathcal{X}_{k+1}^A$ and $\mathcal{X}_{k+1}^B$ satisfying~\eqref{eq:reachable-set-recursion-form}, if
    $\mathcal{X}_0^{A} = \mathcal{X}_0^{B}$, $\Psi^{A} = \Psi^{B}$, and $\Psi^{A}$ is computed using the sharp union~\eqref{eq:sharp_union_identity}, then $\forall k:~\mathrm{CR}(\mathcal{X}_{k+1}^{A}) \subseteq \mathrm{CR}(\mathcal{X}_{k+1}^{\mathrm{B}})$.
    \begin{proof}
        $\mathrm{CR}(\Psi^{\mathrm{A}}) = \mathrm{CH}(\Psi^{\mathrm{A}})$ using~\eqref{eq:sharp_union_identity}, so $\mathrm{CR}(\Psi^{\mathrm{A}}) \subseteq \mathrm{CR}(\Psi^{\mathrm{B}})$. From Lemmas~\ref{lemma:containment-preserved} and \ref{lemma:convex-relaxation-commutes}, $\mathrm{CR}(\mathcal{X}_{k+1}^{A}) \subseteq \mathrm{CR}(\mathcal{X}_{k+1}^{\mathrm{B}})$.
    \end{proof}
\end{theorem}}

\revres{2}{Thm.~\ref{thm:pwa-reachability-tightness} states that the proposed PWA system reachability calculations, when using the sharp union identity, produce hybrid zonotopes with tighter convex relaxations than any other equivalent reachability calculations satisfying~\eqref{eq:reachable-set-recursion-form}.}

\revres{2}{
MLD systems propogate according to
\begin{subequations}
\begin{align}
    &\mathbf{x}_{k+1} = A_x \mathbf{x}_k + B_u \mathbf{u}_k + B_w \mathbf{w}_k + \mathbf{b}_{\mathrm{aff}} \;, \\
    &\mathrm{s.t.}~E_x \mathbf{x}_k + E_u \mathbf{u}_k + E_w \mathbf{w}_k \leq \mathbf{e}_{\mathrm{aff}} \;,
\end{align}
\end{subequations}
where $\mathbf{x}_k$, $\mathbf{u}_k$, and $\mathbf{w}_k$ may contain both continuous and binary variables. In~\cite{bird2023hybrid}, hybrid zonotopes are shown to exactly compute reachable sets for MLD systems using the recursion
\begin{equation} \label{eq:mld-recursion}
    \mathcal{X}_{k+1} = \begin{bmatrix} I & 0 \end{bmatrix} \left[ \left( \begin{bmatrix} A_x \\ E_x \end{bmatrix} \mathcal{X}_k \oplus \mathcal{V} \right) \cap_{[0~I]} \mathcal{H} \right] \;,
\end{equation}
where
\begin{equation}
    \mathcal{V} = \begin{bmatrix} B_u \\ E_u \end{bmatrix} \mathcal{U} \oplus \begin{bmatrix} B_w \\ E_w \end{bmatrix} \mathcal{W} \oplus \begin{bmatrix} \mathbf{b}_{\mathrm{aff}} \\ \mathbf{0} \end{bmatrix} \;,
\end{equation}
and $\mathcal{H} = \{ \mathbf{h}~|~\mathbf{h} \leq \mathbf{e}_{\mathrm{aff}}\}$.}

\revres{2}{The following corollary shows that proposed PWA system reachability calculations produce hybrid zonotopes with tighter convex relaxations than equivalent sets produced using MLD system reachability calculations.
\begin{corollary}
    For $\mathcal{X}_{k+1}^{\mathrm{PWA}}$ computed using~\eqref{eq:gof_recursion_orig} and \eqref{eq:sharp_union_identity}, and $\mathcal{X}_{k+1}^{\mathrm{MLD}}$ computed using \eqref{eq:mld-recursion}, if $\mathcal{X}_0^{\mathrm{PWA}} = \mathcal{X}_0^{\mathrm{MLD}}$, then   
    $\forall k:~\mathrm{CR}(\mathcal{X}_{k+1}^{\mathrm{PWA}}) \subseteq \mathrm{CR}(\mathcal{X}_{k+1}^{\mathrm{MLD}})$.
    \begin{proof}
    Eq.~\eqref{eq:gof_recursion_orig} trivially satisfies~\eqref{eq:reachable-set-recursion-form}. To see that \eqref{eq:mld-recursion} satisfies~\eqref{eq:reachable-set-recursion-form}, define
    \begin{multline}
        \Psi^{\mathrm{MLD}} = \begin{bmatrix}
            I & 0 & 0 & 0 \\
            0 & I & 0 & 0 \\
            0 & 0 & I & 0
        \end{bmatrix} \left[ \left( \begin{bmatrix}
            I & 0 \\
            0 & I \\
            A_x & B_u \\
            E_x & E_u
        \end{bmatrix} \mathcal{D} \oplus \begin{bmatrix}
            0 \\ 0 \\ B_w \\ E_w
        \end{bmatrix} \mathcal{W} \right. \right. \\
        \left. \left.\oplus \begin{bmatrix}
            \mathbf{0} \\ \mathbf{0} \\ \mathbf{b}_{\mathrm{aff}} \\ 
            \mathbf{0}
        \end{bmatrix} \right) \cap_{[0~0~0~I]} \mathcal{H} \right] \;.
    \end{multline}
    Halfspace intersections do not preserve sharpness of hybrid zonotopes~\cite{glunt2025sharp}, so $\Psi^\mathrm{MLD}$ is not guaranteed to be sharp, whereas a set representation $\Psi^\mathrm{PWA}$ constructed using~\eqref{eq:sharp_union_identity} is. Then by Thm.~\ref{thm:pwa-reachability-tightness}, $\mathrm{CR}(\mathcal{X}_{k+1}^{\mathrm{PWA}}) \subseteq \mathrm{CR}(\mathcal{X}_{k+1}^{\mathrm{MLD}})$.
    \end{proof}
\end{corollary}}

\subsection{Memory Complexity of Reachable Set Representations}
\revmargin{2}{R2.3}\revres{2}{
We now analyze the memory complexity of the proposed PWA system reachability calculations. For comparison, results for MLD system calculations using~\eqref{eq:mld-recursion} are provided as well. We follow the approach of~\cite{robbins2025sparsity}, where all problem matrices are assumed to be dense in order to provide an upper bound. To simplify results, $\mathcal{X}_0$ and $\mathcal{U}$ are taken to be zonotopes. The MLD system auxiliary input set $\mathcal{W}$ is a hybrid zonotope.}

\revres{2}{
The state dimension is $n_x$, and $n_{G,0}$ and $n_{G,u}$ are the numbers of generators in $\mathcal{X}_0$ and $\mathcal{U}$, respectively. For hybrid zonotopes $\mathcal{W}$ and $\Psi$, the numbers of continuous generators, binary generatos, and constraints are $n_{Gc,\cdot}$, $n_{Gb,\cdot}$, and $n_{C,\cdot}$, respectively. The combined generator and constraint matrices $G = [G_c~G_b]$ and $A = [A_c~A_b]$ for $\Psi$ are denoted $G_{\psi}$ and $A_{\psi}$. We additionally define the matrix $G_{\psi,+} = [0~0~I] G_{\psi}$. The memory complexity of $\Psi$ for each union identity is given in the Appendix.}

\revres{2}{
The number of inequality constraints in~\eqref{eq:mld-recursion} is $n_e$. For convenience, we define $n_{G,\mathrm{uw}} = n_{G,u} + n_{Gc,w} + n_{Gb,w}$.}

\revres{2}{
The memory complexity of hybrid zonotopes $\mathcal{X}_N$ computed using PWA and MLD system calculations is provided in Tab.~\ref{tab:memory-complexity}. Extending these results to lifted sets $\mathcal{Z}_N$ is straightforward. For the proposed PWA system reachability calculations, the memory complexity is $\mathcal{O}(N)$. This is an improvement over the MLD system case, which has $\mathcal{O}(N^2)$ memory complexity in general. We note that in cases where $A_x=0$, the MLD system reachability calculations reduce to $\mathcal{O}(N)$.}

\setlength{\tabcolsep}{7pt}
\begin{table*}[t]
    \revres{2}{
    \caption{Memory complexity upper bounds for $N$-step reachable sets $\mathcal{X}_N$ using MLD system reachability calculations and the proposed PWA system reachability calculations.}
    \centering
    \begin{tabular}{c|c c c c c}
        \toprule
         method & $n_{Gc}$ & $n_{Gb}$ & $n_C$ & $\nnz{G}$ & $\nnz{A}$ \\ \midrule 
         PWA & $N n_{Gc,\psi} + n_{G,0}$ & $N n_{Gb,\psi}$ & $N (n_{C,\psi} + n_x)$ & $\mathrm{nnz}(G_{\psi,+})$ & $N (2\mathrm{nnz}(G_{\psi,+}) + \mathrm{nnz}(A_{\psi})) + n_x n_{G,0}$ \\ \midrule
          MLD & $N (n_{G,u} + n_{Gc,w} + n_e)$ & $N n_{Gb,W}$ & $N (n_{C,w} + n_e)$ & $N n_x n_{G,\mathrm{uw}}$ & $N^2 \left( \frac{n_e n_{G,\mathrm{uw}}}{2} \right) + N (n_{C,w} (n_{Gc,w} + n_{Gb,w})$ \\  
          & $+ n_{G,0}$ & & & $+ n_x n_{G,0}$ & $ + 2n_e + n_{G,\mathrm{uw}}/2 + n_{G,0})$
         \end{tabular}
    \label{tab:memory-complexity}}
\end{table*}

\subsection{Numerical Example}

To demonstrate the effectiveness of the proposed approach to hybrid system reachability, we consider the two-equilibrium system example from~\cite{bird2023hybrid}. The PWA dynamics are given as
\begin{equation}
    \mathbf{x}_{k+1} = \begin{cases}
        \begin{bmatrix}
            0.75 & 0.25 \\
            -0.25 & 0.75
        \end{bmatrix} \mathbf{x}_k + \begin{bmatrix}
            -0.25 \\ -0.25
        \end{bmatrix}, & x_1 \leq 0, \\
         \begin{bmatrix}
            0.75 & -0.25 \\
            0.25 & 0.75
        \end{bmatrix} \mathbf{x}_k + \begin{bmatrix}
            0.25 \\ -0.25
        \end{bmatrix}, & \text{otherwise},
    \end{cases}
\end{equation}
where the initial state set is
\begin{equation}
    \mathcal{X}_0 = \left\langle \begin{bmatrix}
        0.25 & -0.19 \\
        0.19 & 0.25
    \end{bmatrix}, \begin{bmatrix}
        -1.31 \\ 2.55
    \end{bmatrix} \right\rangle \;.
\end{equation}

For constrained reachability calculations, $\forall k \in \{1, ..., 15\}$, the state constraint set is taken to be 
\begin{subequations}
\begin{align}
    \mathcal{F}_k &= \left\{ \begin{bmatrix} x_1 \\ x_2 \end{bmatrix} \revmargin{3}{R3.2}~\middle|~ x_1 \in [-2, 2],\; x_2 \in [-1, 3] \right\} \;, \\
    &= \left\langle \begin{bmatrix}
        2 & 0 \\
        0 & 2
    \end{bmatrix}, \begin{bmatrix}
        0 \\ 1
    \end{bmatrix} \right\rangle\;.
\end{align}
\end{subequations}

Reachable sets $\mathcal{X}_k$ computed using the proposed PWA reachability calculations with hybrid zonotopes are shown in Fig.~\ref{fig:two_equilibrium_reach_sets}. 
In Fig.~\ref{fig:two_equilibrium_convex_relaxations}, \revmargin{2}{R2.4}\revres{2}{the proposed reachability calculations are shown to produce hybrid zonotopes with much tighter convex relaxations than those produced using MLD-based calculations~\cite{bird2023hybrid}.} 

\begin{figure}[t]
    \centering
    \input{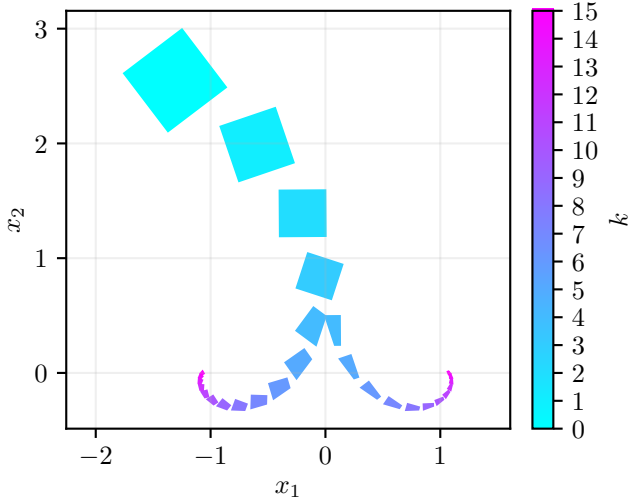}
    \caption{Reachable sets of two-equilibrium system computed using \eqref{eq:union_of_gofs} and \eqref{eq:pwa_single_mode_GOF}, \eqref{eq:gof_recursion_orig}.}
    \label{fig:two_equilibrium_reach_sets}
\end{figure}

\begin{figure}[t]
    \centering
    \input{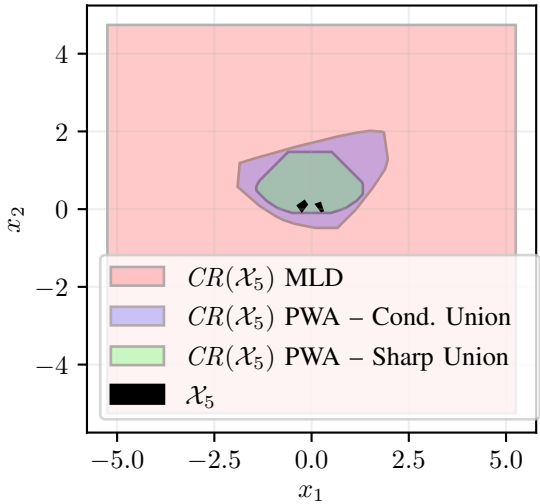}
    \caption{\revmargin{2}{R2.3}Convex relaxations for the 5-step reachable set $\mathcal{X}_5$ of the two-equilibrium system. Note that $\mathit{CR}(\mathcal{X}_5) \neq \mathit{CH}(\mathcal{X}_5)$ for the PWA -- Sharp Union case because the generalized intersection in~\eqref{eq:gof_recursion_orig} does not preserve sharpness.}
    \label{fig:two_equilibrium_convex_relaxations}
\end{figure}

The hybrid zonotope memory complexities for the 15-step reachable sets of this system are given in Tab.~\ref{tab:two_equilibrium_complexity}. 
The lifted and constrained reachability calculations needed to form optimal planning problems incur a modest memory complexity increase over their unconstrained counterparts due to the inclusion of $\overline{\mathcal{S}}$ in~\eqref{eq:lifted-gof-recursion} and $\mathcal{F}_{k+1}$ in~\eqref{eq:gof-constrained}. In both cases, the memory complexity is $\mathcal{O}(N)$. \revmargin{2}{R2.3}\revres{2}{For this system, $A_x=0$ in the MLD formulation so memory complexity is also $\mathcal{O}(N)$}.

The discrepancy in binary variables between MLD and PWA calculations results from the choose-one constraints in the hybrid zonotope union identities presented in the appendix. 
 
In this example, the proposed PWA reachability calculations improve on equivalent MLD calculations both in tightness and in representation size, both of which are known to be very important for efficient mixed-integer optimization~\cite{marcucci2019mixed}.

\setlength{\tabcolsep}{5pt}
\begin{table}
    \caption{Memory complexity of unconstrained reachable sets $\mathcal{X}_{15}$ (Def.~\ref{def:reachable-set}) and lifted, constrained reachable sets $\mathcal{Z}_{15}$ (Def.~\ref{def:Z_def}).}
    \centering
    \begin{tabular}{c|c|c c c c c c}
        \toprule
         set & method & $n$ & $n_{Gc}$ & $n_{Gb}$ & $n_C$ & $\mathrm{nnz}(G)$ & $\mathrm{nnz}(A)$ \\ \midrule 
          $\mathcal{X}_{15}$ & MLD & 2 & 182 & 15 & 150 & 2 & 692 \\ \midrule
          $\mathcal{X}_{15}$ & PWA cond. & 2 & 92 & 30 & 75 & 12 & 442 \\ \midrule
          $\mathcal{X}_{15}$ & PWA sharp & 2 & 122 & 30 & 105 & 12 & 502 \\ \midrule
          $\mathcal{Z}_{15}$ & PWA cond. & 32 & 152 & 30 & 135 & 34 & 722 \\ \midrule
          $\mathcal{Z}_{15}$ & PWA sharp & 32 & 182 & 30 & 165 & 34 & 782
         \end{tabular}
    \label{tab:two_equilibrium_complexity}
\end{table}

\section{ADMM-FP Mixed Integer Programming Heuristic} \label{sec:admm-fp}

Sec.~\ref{sec:hz_reach} developed a flexible method to construct planning problems for hybrid systems using hybrid zonotope reachability analysis. The planning problem is of the form~\eqref{eq:optim_prob_gen} where $\mathcal{Z}_N$ is a hybrid zonotope as defined in Def.~\ref{def:Z_def}.

This section presents an algorithm that 
searches for a feasible $\mathbf{z} \in \mathcal{Z}_N$, while attempting to minimize the objective function, for the optimization problem~\eqref{eq:optim_prob_gen}.
To ensure computational tractability on resource-constrained embedded systems, we avoid traditional branch-and-bound solution approaches, which have high memory requirements in general as they rely on queue data structures to store convex sub-problems. Instead, we propose a new MIP heuristic that integrates elements of ADMM and feasibility pump heuristics. This approach is simple to implement, can be effectively warm-started, and has bounded memory usage. Further, it makes efficient use of the hybrid zonotope structure.

\subsection{ADMM for Hybrid Zonotopes}
Here, we briefly describe an ADMM heuristic to find a feasible $\mathbf{z} \in \mathcal{Z}_N$, for the optimization problem~\eqref{eq:optim_prob_gen}. This \anon{presentation is similar to related work}{presentation extends our prior work} on convex optimization where the feasible set is a constrained zonotope~\cite{robbins2025sparsity}. 

Substituting the hybrid zonotope definition~\eqref{eq:hybzono-definition} for $\mathcal{Z}_N$ into~\eqref{eq:optim_prob_gen} produces the optimization problem
\begin{subequations} \label{eq:miqp-hybzono-expanded}
\begin{align}
    &\min_{\bm{\xi}} \frac{1}{2} \bm{\xi}^T G^T P G \bm{\xi} + \left(G^T (P \mathbf{c} + \mathbf{q})\right)^T \bm{\xi} \;, \\
    &\mathrm{s.t.} \; A \bm{\xi} = \mathbf{b},\; \bm{\xi} \in \mathcal{B}_{\mathrm{MI}} \;,
\end{align}
\end{subequations}
where $G = \begin{bmatrix} G_c & G_b \end{bmatrix}$, $A = \begin{bmatrix} A_c & A_b \end{bmatrix}$, $\bm{\xi} = \begin{bmatrix} \bm{\xi}_c^T & \bm{\xi}_b^T \end{bmatrix}^T$. The mixed integer box is defined as
\begin{equation} \label{eq:B_MI}
    \mathcal{B}_{\mathrm{MI}} = \begin{cases}
        [-1,1]^{n_{Gc}} \times \{-1, 1\}^{n_{Gb}} ,& \text{canonical form,} \\
        [0,1]^{n_{Gc}} \times \{0, 1\}^{n_{Gb}} ,& \text{01-form.}
    \end{cases}
\end{equation}

Defining 
\begin{equation} \label{eq:tildeP-tildeq-def}
    \tilde{P} = G^T P G \;,\; \tilde{\mathbf{q}} = G^T (P \mathbf{c} + \mathbf{q}) \;,
\end{equation}
and neglecting constant terms, \eqref{eq:miqp-hybzono-expanded} is seen to be a mixed integer quadratic program (MIQP). The constraints are of the form $\bm{\xi} \in \mathcal{A} \cap \mathcal{B}_{MI}$ where $\mathcal{A} = \{\bm{\xi} \revmargin{3}{R3.2}~|~ A \bm{\xi} = \mathbf{b} \}$ is an affine set.

Problem~\eqref{eq:miqp-hybzono-expanded} can be written in the form of~\eqref{eq:admm-general} for solution via ADMM by introducing optimization variables $\bm{\zeta}$. Using indicator functions, \eqref{eq:miqp-hybzono-expanded} then becomes
\begin{subequations} \label{eq:miqp-admm-problem}
\begin{align}
&\min_{\bm{\xi}, \bm{\zeta}} \frac{1}{2} \bm{\xi}^T \tilde{P} \bm{\xi} + \tilde{\mathbf{q}}^T \bm{\xi} + I_{\mathcal{A}}(\bm{\xi}) + I_{\mathcal{B}_{\mathrm{MI}}}(\bm{\zeta}) \;, \\
&\mathrm{s.t.}\;\bm{\xi} = \bm{\zeta} \;.
\end{align}
\end{subequations}

Define the matrix $M$ to be
\begin{equation} \label{eq:M-def}
    M = \begin{bmatrix}
    \tilde{P} + \rho I & A^T \\
    A & 0
\end{bmatrix} \;.
\end{equation}

The ADMM iterations~\eqref{eq:admm-iterations-general} then reduce to
\begin{subequations} \label{eq:mi-admm-iterations-with-obj}
\begin{align}
&\bm{\xi}_{k+1} = \begin{bmatrix} I & 0 \end{bmatrix} M^{-1} \begin{bmatrix}
    -\tilde{\mathbf{q}} + \rho (\bm{\zeta}_k - \mathbf{u}_k) \\
    \mathbf{b}
\end{bmatrix} \;,  \\
&\bm{\zeta}_{k+1} = \pi_{\mathcal{B}_{\mathrm{MI}}}(\bm{\xi}_{k+1} + \mathbf{u}_k) \;,  \\
&\mathbf{u}_{k+1} = \mathbf{u}_{k} + \bm{\xi}_{k+1} - \bm{\zeta}_{k+1}  \;.
\end{align}
\end{subequations}

We assume that $A$ is full row rank such that $M$ is invertible. If $A$ is not full row rank, then redundant constraints can be removed as in~\cite[Alg. 1]{vinod2025projection} or using QR decomposition as in~\cite{bird2022hybrid}.

\revmargin{1}{R1.6}\revres{1}{We note that the mixed-integer box projection $\pi_{\mathcal{B}_{\mathrm{MI}}}(\cdot)$ afforded by the hybrid zonotope structure is $\mathcal{O}(1)$ and trivial to compute.}
The iterations~\eqref{eq:mi-admm-iterations-with-obj} are analogous to the ADMM MIP heuristic presented in~\cite{takapoui2020simple}. Because $\mathcal{B}_{\mathrm{MI}}$ is not a convex set, the iterations~\eqref{eq:mi-admm-iterations-with-obj} violate the assumptions of the ADMM algorithm (Sec.~\ref{sec:admm-prelims}) and are not guaranteed to converge. As pointed out in~\cite{takapoui2020simple}, these iterations may still produce approximately feasible---and often approximately optimal---solutions in many cases.

Motivated by the observation in~\cite{takapoui2020simple} that the ADMM MIP heuristic is more likely to converge for larger values of the ADMM penalty parameter $\rho$, we additionally consider the case 
where the objective in~\eqref{eq:optim_prob_gen} is neglected (i.e., $\rho \rightarrow \infty$). Eq.~\eqref{eq:miqp-admm-problem} then becomes
\begin{subequations} \label{eq:miqp-admm-problem-no-obj}
\begin{align}
&\min_{\bm{\xi}, \bm{\zeta}} I_{\mathcal{A}}(\bm{\xi}) + I_{\mathcal{B}_{\mathrm{MI}}}(\bm{\zeta}) \;, \\
&\mathrm{s.t.}\;\bm{\xi} = \bm{\zeta} \;.
\end{align}
\end{subequations}

Referencing~\eqref{eq:admm-iterations-general}, the ADMM iterations for~\eqref{eq:miqp-admm-problem-no-obj} reduce to
\begin{subequations} \label{eq:mi-admm-iterations-no-obj}
\begin{align}
&\bm{\xi}_{k+1} = \pi_{\mathcal{A}}(\bm{\zeta}_k - \mathbf{u}_k) \;, \\
&\bm{\zeta}_{k+1} = \pi_{\mathcal{B}_{\mathrm{MI}}}(\bm{\xi}_{k+1} + \mathbf{u}_k) \;,  \\
&\mathbf{u}_{k+1} = \mathbf{u}_{k} + \bm{\xi}_{k+1} - \bm{\zeta}_{k+1}  \;.
\end{align}
\end{subequations}

In~\anon{\cite{robbins2025sparsity}, the authors show that}{our prior work~\cite{robbins2025sparsity}, we show that} ADMM iterations for~\eqref{eq:optim_prob_gen} where $\mathcal{Z}_N$ is a constrained zonotope make efficient use of the constrained zonotope structure. In particular, \anon{they}{we} argue that $M$ will often have fewer non-zero elements using constrained zonotopes than an equivalent formulation using halfspace representation (H-rep) polytopes, resulting in reduced factorization time and faster iterations. Additionally, as the constrained zonotope factors are normalized by definition, \anon{they}{we} do not employ any preconditioning algorithm. These arguments also apply to the hybrid zonotope case presented here.

\subsection{ADMM-FP Heuristic}
The ADMM iterations are reminiscent of the iterations of the feasibility pump algorithm presented in Sec.~\ref{sec:fp}. For both~\eqref{eq:mi-admm-iterations-with-obj} and \eqref{eq:mi-admm-iterations-no-obj}, the iterates $\bm{\xi}_k$ belong to the affine set $\mathcal{A}$, while the iterates $\bm{\zeta}_k$ belong to the mixed-integer box $\mathcal{B}_{\mathrm{MI}}$. In cases where the ADMM heuristic converges, these iterates are driven to each other such that $\bm{\xi}_k \approx \bm{\zeta}_k \in \mathcal{A} \cap \mathcal{B}_{\mathrm{MI}}$. For the feasibility pump, the iterates $\mathbf{z}^*$ belong to the LP polyhedron $\mathcal{C}$, while the iterates $\tilde{\mathbf{z}}$ belong to the mixed-integer box $\mathcal{R}$. When the heuristic converges, $\mathbf{z}^* \approx \tilde{\mathbf{z}} \in \mathcal{C} \cap \mathcal{R}$. 

The ADMM iterations are generally much more efficient than the feasibility pump iterations as they only require matrix back substitutions, matrix multiplications, and projections onto a mixed-integer box, while the feasibility pump solves an LP at every iteration. ADMM also uses dual variables $\mathbf{u}_k$ to guide its search, which, at least in the context of convex optimization, can make it more efficient than related methods (e.g., alternating projection methods) that do not take advantage of the dual space~\cite{boyd2011distributed}. A key difference between the ADMM iterations~\eqref{eq:mi-admm-iterations-with-obj} and \eqref{eq:mi-admm-iterations-no-obj} and the feasibility pump is that the ADMM iterates $\bm{\xi}_k$ are only feasible with respect to the problem equality constraints, while the feasibility pump iterates $\mathbf{z}^*$ are feasible with respect to both the equality and inequality constraints. The ADMM iterations enforce the inequality constraints and integrality constraints simultaneously via projection onto $\mathcal{B}_{\mathrm{MI}}$.

Empirically, we found that in cases where the ADMM iterations fail to converge, the iterates $(\bm{\xi}_k, \bm{\zeta}_k, \mathbf{u}_k)$ exhibit cycling behavior, i.e., $(\bm{\xi}_k, \bm{\zeta}_k, \mathbf{u}_k)$ can approximately repeat for different values of $k$. For the feasibility pump, cycling is very common and is mitigated using random perturbation and restart procedures (see Sec.~\ref{sec:fp}). As such, taking inspiration from the feasibility pump, we propose random perturbation and restart augmentations to the ADMM iterations. 

\subsubsection{Algorithm}
\begin{algorithm}
\caption{ADMM-FP for hybrid zonotopes \\
Inputs: $\mathcal{Z}_N$, $P$, $\mathbf{q}$, $\bm{\zeta}^*, \mathbf{u}^*$ \\
Outputs: $\mathbf{z} \in \mathcal{Z}_N$}
\begin{algorithmic}[1]
    \State factorize $M$ and $A A^T$
    \State $(\mathrm{phase}, k_{\mathrm{max}}, k) \gets (1, k_{ph1}, 0)$ 
    \State $(r_p, r_-, k_r) \gets (\infty, \infty, 0)$
    \State $(\bm{\xi}_0, \bm{\zeta}_0, \mathbf{u}_0) \gets (\bm{\zeta}^*, \bm{\zeta}^*, \mathbf{u}^*)$ 
    \While{$r_p > \epsilon_p$ \textbf{and} $k < k_{\mathrm{max}}$}
        \If{phase = 1}
            \State $\bm{\xi}_{k+1} \gets \begin{bmatrix} I & 0 \end{bmatrix}        M^{-1} \begin{bmatrix} -\tilde{\mathbf{q}} + \rho (\bm{\zeta}_k - \mathbf{u}_k) \\
            \mathbf{b}
            \end{bmatrix}$
        \Else \Comment{phase 2}
            \State $\bm{\xi}_{k+1} \gets \pi_{\mathcal{A}} (\bm{\zeta}_k - \mathbf{u}_k)$ \Comment{uses $AA^T$ factorization}
        \EndIf
        \State $\bm{\zeta}_{k+1} \gets \pi_{\mathcal{B}_{\mathrm{MI}}}(\bm{\xi}_{k+1} + \mathbf{u}_k)$
        \State $\mathbf{u}_{k+1} \gets \mathbf{u}_{k} + \bm{\xi}_{k+1} - \bm{\zeta}_{k+1}$ 
        \State $r_p \gets ||\bm{\xi}_{k+1} - \bm{\zeta}_{k+1}||_{\infty}$ \Comment{primal residual norm}
        \LineComment{random perturbation if cycle detected} \label{alg:line:admm-fp-cycle-detected}
        \If {cycle detected}
            \State \Call{binflip}{$\bm{\xi_{k+1}, \zeta_{k+1}, \text{PERTURB}}$}
        \EndIf
        \LineComment{restart if failing to make progress} \label{alg:line:admm-fp-restart}
        \If{$r_p < r_-$}
            \State $(k_r, r_-) \gets (0, r_p)$
        \Else
            \State $k_r \gets k_r + 1$
        \EndIf
        \If {$k_r \geq k_{\mathrm{restart}}$}
            \State \Call{binflip}{$\bm{\xi_{k+1}, \zeta_{k+1}, \text{RESTART}}$}
            \State $(k_r, r_-) \gets (0, r_p)$
        \EndIf
        \LineComment{switch to phase 2 if max phase 1 iterations reached}
        \State $k \gets k+1$
        \If {$k = k_{\mathrm{max}}$ \textbf{and} phase = 1}
            \State $(\mathrm{phase}, k_{\mathrm{max}}, k) \gets (2, k_{ph2}, 0)$ 
        \EndIf
    \EndWhile
    \State $\mathbf{z} \gets G \bm{\zeta}_{k} + \mathbf{c}$
    \State \Return $\mathbf{z}$
\end{algorithmic}
\label{alg:admm-fp}
\end{algorithm}

The proposed MIP heuristic is given in Algorithm~\ref{alg:admm-fp}. Throughout the rest of this paper, we refer to this heuristic as ADMM-FP.
When a cycle is detected in the ADMM iterations (line~\ref{alg:line:admm-fp-cycle-detected}), we randomly flip binary elements in $\bm{\zeta}_k$ with probably based on the fractionality of the corresponding element in $\bm{\xi}_k$. Specifically, we define the fractionality of the iterates $\bm{\xi}_k$, $\bm{\zeta}_k$ at index $j$ to be 
\begin{equation} \label{eq:fractionality-def}
    f = |\xi_{j}-\zeta_{j}|/(\overline{\beta}_j - \underline{\beta}_j) \;,
\end{equation}
where $\overline{\beta}_j$ and $\underline{\beta}_j$ are the upper and lower bounds of~\eqref{eq:B_MI} at index $j$. These quantities are defined in order to maintain compatibility with the canonical hybrid zonotope definition. For the case that we are using the 01-form, \eqref{eq:fractionality-def} reduces to $f = |\xi_{j}-\zeta_{j}|$. 

We use $f$ here as the probability that $\zeta_j$ will be flipped. That is, if $f=0$, $\zeta_j$ will not be flipped, and if $f=0.5$, then $\zeta_j$ has a 50\% chance of being flipped. This perturbation method is very conservative when the ADMM is nearly converged, and more aggressive when far from convergence. 
Empirically, we found this perturbation method to be more effective for the ADMM heuristic than the corresponding perturbation procedure in the feasibility pump (see line~\ref{alg:line:fp-perturb} in Algorithm~\ref{alg:fp}). Note that we also use this perturbation method for cycles of length greater than one, in contrast with the feasibility pump which applies a restart procedure when a cycle of length greater than one is detected. Cycles of length greater than one are common for the ADMM heuristic, and can often be broken with the more conservative perturbation method outlined here.

For restarts (line~\ref{alg:line:admm-fp-restart}), we adopt the procedure from the feasibility pump, which is given in line~\ref{alg:line:fp-restart} of Algorithm~\ref{alg:fp}. Rather than apply the restart procedure upon detecting a cycle of length greater than one, we restart if the algorithm fails to make progress in reducing the primal residual norm $r_p$ over $k_{\mathrm{restart}}$ iterations. The random perturbation and restart procedures are given in Algorithm~\ref{alg:perturb-binaries}. 

Drawing inspiration from the OFP, we implement a two-phase approach where the objective function is used to guide the search in phase 1, and only feasibility is considered in phase 2. Phase 1 uses the iterations~\eqref{eq:mi-admm-iterations-with-obj}, and phase 2 uses~\eqref{eq:mi-admm-iterations-no-obj}. Alternatively, a single-phase implementation using iterations~\eqref{eq:mi-admm-iterations-with-obj} with a gradually increasing ADMM penalty parameter $\rho$ could have been considered. We favor the two-phase approach in our implementation because changing $\rho$ values would necessitate refactorizing $M$, which can be computationally expensive; in general, factorizing $M$ scales as $\mathcal{O}(m^3)$ where $M \in \real^{m \times m}$.

Algorithm~\ref{alg:admm-fp} is said to converge when the infinity norm of the primal residual $r_p = ||\bm{\xi}_k -\bm{\zeta}_k||_{\infty} < \epsilon_p$, where $\epsilon_p$ is a feasibility tolerance.
The values $\bm{\zeta}^*, \mathbf{u}^*$ are the initial primal and dual iterates for the algorithm. 
For cycle detection, instead of checking whether the iterates $(\bm{\xi}_k, \bm{\zeta}_k, \mathbf{u}_k)$ have previously been visited, we instead use the primal residual norm $r_p$ as a proxy. Specifically, we use a circular buffer with length $l_{\mathrm{buf}}$ to store $r_p$ values, and check for repeats within a specified numerical tolerance $\epsilon_{\mathrm{buf}}$.

In general, the ADMM-FP parameters may need to be tuned for a given problem class. The most important parameters for performance of the algorithm are $\rho$, which is used to balance optimality versus feasibility during phase 1 of the heuristic, and $\epsilon_p$, which is the numerical tolerance for accepting a solution as feasible.

\begin{algorithm}
\caption{Binary variable flipping for perturbation and restart operations}
\begin{algorithmic}[1]
    \Procedure{binflip}{$\bm{\xi}, \bm{\zeta}, \mathrm{mode}$}
    \For{$j \in \{n_{Gc}, ..., n_{Gc}+n_{Gb}\}$} \Comment{binary variables}
        \State $f \gets |\xi_{j}-\zeta_{j}|/(\overline{\beta}_j - \underline{\beta}_j)$
        \If{mode = PERTURB}
            \State $\mathrm{flip} \gets \mathrm{rand}(0,1) < f$
        \ElsIf{mode = RESTART}
            \State $\mathrm{flip} \gets  f + \max\left(\mathrm{rand}(-0.3, 0.7), 0\right) > 0.5$
        \EndIf
        \If {flip}
            \State $\zeta_{j} = \begin{cases}
                \underline{\beta}_j \;,\; \zeta_{j} = \overline{\beta}_j \\
                \overline{\beta}_j \;,\; \zeta_{j} = \underline{\beta}_j
            \end{cases}$
        \EndIf
    \EndFor
    \EndProcedure
\end{algorithmic}
\label{alg:perturb-binaries}
\end{algorithm}

\subsubsection{Warm-Starting} \label{sec:warm-start}
The performance of Algorithm~\ref{alg:admm-fp} depends strongly on the choice of initial iterates $\bm{\zeta}^*, \mathbf{u}^*$. The ADMM heuristic of~\cite{takapoui2020simple} uses random initial iterates. The authors propose running the heuristic multiple times with different initial iterates until an acceptable solution is achieved. The feasibility pump, by contrast, uses the solution to the convex relaxation of the mixed-integer program as its initial iterate $\mathbf{z}^*$~\cite{fischetti2005feasibility}. Our approach mirrors the feasibility pump, and by default $\bm{\zeta}^*, \mathbf{u}^*$ are computed by solving the convex relaxation of the MIQP. The feasible set $\mathit{CR}(\mathcal{Z}_N)$ then becomes a constrained zonotope, and the ADMM method of~\cite{robbins2025sparsity} is used to produce an initial $\bm{\zeta}^*, \mathbf{u}^*$. When solving this QP, the same ADMM parameters are used as in the ADMM-FP heuristic where applicable. 

The effectiveness of using the solution to the convex relaxation as the initial iterates for Algorithm~\ref{alg:admm-fp} depends on the tightness of $\mathcal{Z}_N$. In the extreme case that $\mathcal{Z}_N$ is sharp, $G \bm{\zeta}^* + \mathbf{c} \in \mathit{CH}(\mathcal{Z}_N)$, and the initial iterates can be expected to be reasonably close to a high-quality, feasible solution to the MIQP. As shown in Sec.~\ref{sec:hz_reach}, the proposed PWA system reachability calculations produce feasible sets $\mathcal{Z}_N$ with tighter convex relaxations than corresponding sets produced using an MLD system description. Use of the sharp union identity (App.~\ref{sec:sharp-union}) can be expected to produce initial iterates which are closer to $\mathcal{Z}_N$ when compared to the condensed union (App.~\ref{sec:condensed-union}) at the expense of a higher-complexity representation of $\mathcal{Z}_N$.

In some cases, we may have a point $\mathbf{z}^*$ that we expect to be reasonably close to a feasible point $\mathbf{z} \in \mathcal{Z}_N$. For example, $\mathbf{z}^*$ could be derived from a previous motion plan or from another planning methodology (e.g., a machine learning-based planner). In either case, $\mathbf{z}^* \notin \mathcal{Z}_N$ in general.

To find warm-start iterates $\bm{\zeta}^*, \mathbf{u}^*$ based on a given point $\mathbf{z}^*$, we solve the projection QP
\begin{equation} \label{eq:warmstart-proj}
    \min_{\mathbf{z}} ||\mathbf{z} - \mathbf{z}^*||_2, \;\mathrm{s.t.} \; \mathbf{z} \in \mathit{CR}(\mathcal{Z}_N) \;, 
\end{equation}
using the ADMM method of~\cite{robbins2025sparsity}. The optimal iterates $\bm{\zeta}$ and $\mathbf{u}$ generated when solving this QP are taken as the initial iterates $\bm{\zeta}^*, \mathbf{u}^*$ for Algorithm~\ref{alg:admm-fp}.
Note that the $\mathbf{u}^*$ computed this way corresponds to a different objective than the one specified in~\eqref{eq:miqp-admm-problem}. In practice, however, warm-starting $\mathbf{u}^*$ in this way works well.

\section{Numerical Examples} \label{sec:numerical-results}

The hybrid zonotope-based reachability analysis developed in Sec.~\ref{sec:hz_reach} provides a convenient and numerically efficient method to formulate planning problems for PWA systems. The ADMM-FP mixed-integer programming heuristic developed in Sec.~\ref{sec:admm-fp} is used to search for feasible, low-cost solutions to these planning problems, and makes efficient use of the hybrid zonotope structure. This section empirically evaluates the effectiveness of the proposed methods through four numerical examples: 1) Random hybrid zonotopes representing the feasible set with random linear objectives, 2) random reach-avoid problems for a double integrator system, 3) a benchmark ball and paddle system, and 4) an integrated behavior and motion planning problem for an autonomous vehicle.

In all cases, computations were executed on an Ubuntu 24.04 desktop with an Intel\textregistered~Core\textsuperscript{\tiny TM} i7-14700 × 28 processor and 32~GB of RAM. All algorithms were configured to use a single thread. 
\anon{The proposed ADMM-FP heuristic and hybrid zonotope reachability calculations were implemented in C++, and Python was used for plotting.}{Problems were formulated using the open-source ZonoOpt library\footnote{\url{https://github.com/psu-PAC-Lab/ZonoOpt}} developed by the authors~\cite{robbins2025sparsity}. The proposed ADMM-FP heuristic was implemented in C++ and integrated into the library.}

\subsection{Random MILP} \label{sec:random-milp}

To evaluate the efficacy of the proposed ADMM-FP heuristic \revmargin{1}{R1.4}\revres{1}{independent of the hybrid system planning context}, we first consider the problem of finding a solution to a random MILP where the feasible set is represented as a hybrid zonotope. 

One hundred different random hybrid zonotopes were generated with dimensions $n=100$, $n_{Gc} = 200$, $n_{Gb} = 50$, and $n_C = 50$ (see~\eqref{eq:hybzono-definition}). The $G_c$, $G_b$, $A_c$, and $A_b$ matrices are sparse with a density (proportion of non-zero elements) of $0.1$. Sparse matrices are typical in the planning problem formulations presented in Sec.~\ref{sec:hz_reach}.
The non-zero elements of these matrices, and all elements of the vectors $\mathbf{c}$ and $\mathbf{b}$, were selected as $\mathrm{rand}(-1,1)$. The canonical hybrid zonotope definition is used here.
A random linear objective was used, i.e., $P=0$ in~\eqref{eq:optim_prob_gen}. The elements of the linear objective vector $\mathbf{q}$ were also selected as $\mathrm{rand}(-1,1)$.

We compare the proposed ADMM-FP heuristic to an ADMM heuristic without the proposed modifications, and an OFP heuristic (see Sec.~\ref{sec:fp}). The ADMM heuristic corresponds to phase 1 of Algorithm~\ref{alg:admm-fp} with restart and perturbation procedures removed. The OFP heuristic was implemented in C++ using the HiGHS LP solver~\cite{Huangfu_Parallelizing_the_dual_2018}. The OFP implementation uses an analogous convergence criterion to the ADMM and ADMM-FP heuristics, i.e., the algorithm is converged if $||\mathbf{z}^* - \tilde{\mathbf{z}}||_{\infty} \leq \epsilon_p$ with $\epsilon_p$ the same as in the ADMM-FP and ADMM heuristics.

We additionally provide data for the state-of-the-art mixed-integer solvers Gurobi (version 13.0.2, commercial) and \revmargin[true]{1}{R1.1}\revres{1}{SCIP} (version 10.0.3, open source). Gurobi and SCIP are configured to terminate as soon as a feasible solution is found within the specified numerical tolerance $\epsilon_p$. 
\revmargin{2}{R2.12}\revres{2}{Problem~\eqref{eq:optim_prob_gen} is passed directly to Gurobi, while SCIP requires a quadratic epigraph reformulation because it can handle quadratic constraints, but not quadratic objective function terms.}
Our intention is not to directly compare the proposed ADMM-FP heuristic to Gurobi and SCIP---which are general-purpose global MIP solvers---but to provide a point of reference for readers who are familiar with these solvers and their performance. Comparisons with the OFP and ADMM heuristics are more salient, as these are both primal heuristics suitable for use on resource-constrained embedded systems.

The ADMM-FP and ADMM (where applicable) heuristics were configured to use the default parameters specified in Tab.~\ref{tab:admm-fp-default-params} (with the exception of $k_{ph1}$ and $k_{ph2}$ as discussed below), and the OFP heuristic was configured to use the parameters in Tab.~\ref{tab:ofp-default-params}. The parameter $l_{\mathrm{buf}}$ in Tab.~\ref{tab:ofp-default-params} refers to the length of the buffer used for cycle detection, i.e., it is the longest cycle that can be detected.
A time limit of $t_{\mathrm{max}}=1$~sec was used for all methods. The max iteration parameters ($k_{ph2}$ for ADMM-FP and $k_{ph1}$ for ADMM) were set to large values such that the algorithms will not terminate before reaching $t_{\mathrm{max}}$ unless the numerical convergence criterion is achieved. When solving convex relaxations using~\cite[Alg.~1]{robbins2025sparsity}, the dual convergence tolerance was set to $\epsilon_d=0.01$.

\setlength{\tabcolsep}{5pt}
\begin{table}
    \caption{Default parameter values for ADMM-FP.}
    \centering
    \begin{tabular}{l|c c c c c c c}
        \toprule
         \textbf{parameter} & $\rho$ & $\epsilon_p$ & $k_{\mathrm{restart}}$ & $k_{ph1}$ & $k_{ph2}$ & $l_{\mathrm{buf}}$ & $\epsilon_\mathrm{buf}$  \\ \midrule
         \textbf{value} & 10 & 0.001 & 5000 & 10000 & 90000 & 20 & 0.001
         \end{tabular}
    \label{tab:admm-fp-default-params}
\end{table}

\setlength{\tabcolsep}{7pt}
\begin{table}
    \caption{Default parameter values for OFP~\cite{achterberg2007improving}.}
    \centering
    \begin{tabular}{l|c c c c c}
        \toprule
         \textbf{parameter} & $\alpha_0$ & $\phi$ & $t$ & $\delta_{\alpha}$ & $l_{\mathrm{buf}}$  \\ \midrule
         \textbf{value} & 1 & 0.9 & 20 & 0.005 & 10
         \end{tabular}
    \label{tab:ofp-default-params}
\end{table}

Results for the random MILP task are given in Fig.~\ref{fig:random-milp}. As seen in Fig.~\ref{fig:random-milp}(a), the ADMM-FP and OFP heuristics found feasible solutions for all tested cases, and the ADMM heuristic found a feasible solution in 79 / 100 cases. All three heuristics found good quality solutions as shown in Fig.~\ref{fig:random-milp}(b), with the objective for the vast majority of found solutions being very close to that of the global optimum (computed using Gurobi). In Fig.~\ref{fig:random-milp}(d) the OFP heuristic converged in fewer iterations than the ADMM and ADMM-FP heuristics, but those iterations required the solution of an LP and thus were much slower than the corresponding ADMM iterations. As such, for these random MILPs, Fig.~\ref{fig:random-milp}(c) shows that the ADMM and ADMM-FP heuristics found a feasible solution about an order of magnitude faster than the OFP heuristic on average. Note that these solution times include the time to solve the convex relaxation of the MILP to generate initial iterates. Gurobi and SCIP do not have a direct equivalent to the number of ADMM or OFP iterations and so none is provided in Fig.~\ref{fig:random-milp}(d).

\begin{figure}
\centering
\input{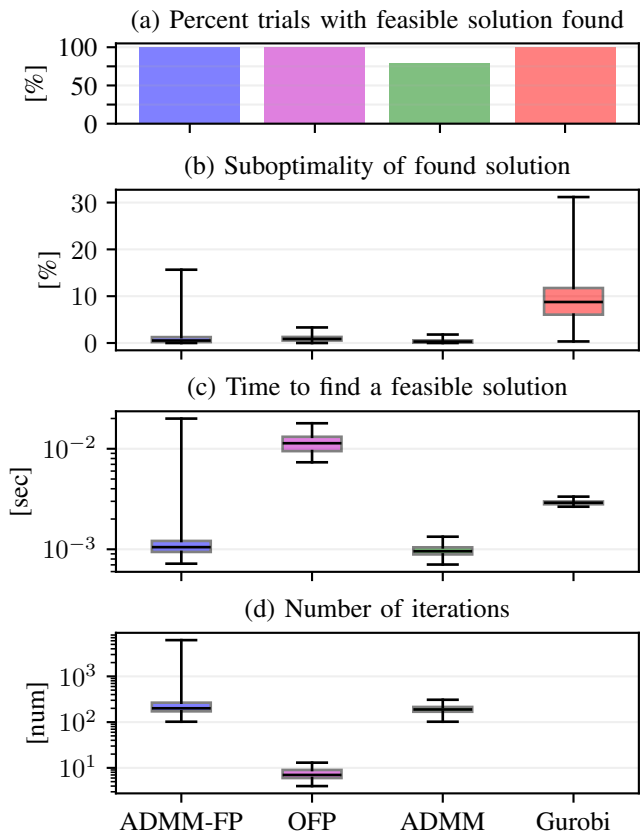}
\caption{Comparison of solution methods for an MILP feasibility problem for which the feasible space is a random, non-empty hybrid zonotope. In these box and whisker plots, the box encloses the middle 50\% of data points, and the whiskers enclose all other data points. The median is depicted as a black line.}
\label{fig:random-milp}
\end{figure}

\subsection{Reach-Avoid Problem} \label{sec:reach-avoid}

\begin{figure}[t]
    \centering
    \input{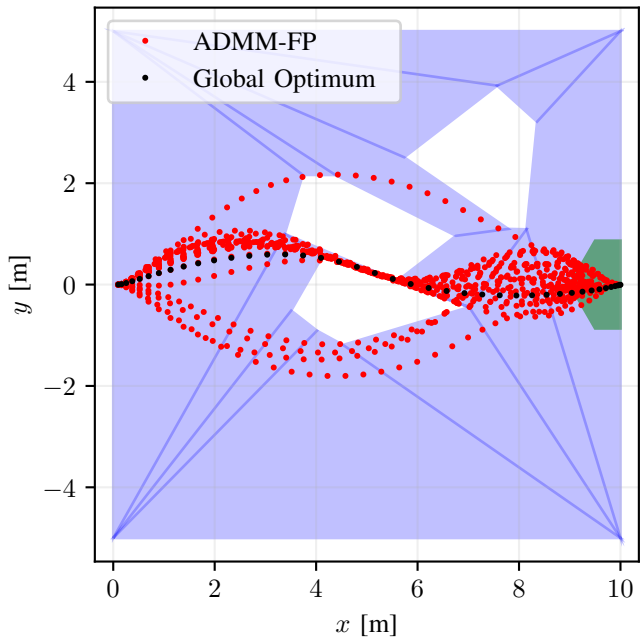}
    \caption{Trajectories generated by ADMM-FP compared to the globally optimal trajectory for a 40-step motion planning problem. The blue set is $\mathcal{P}$ and the green set shows the first two dimensions of $\mathcal{F}_N$.}
    \label{fig:motion-plan-traj}
\end{figure}

\begin{figure}[t]
    \centering
    \input{figs/reach_avoid_partial_solutions.pgf}
    \caption{Trajectories produced by ADMM-FP prior to convergence at the specified numerical tolerance $\epsilon_p=0.001$. Colors indicate the primal residual norm $r_p$ of a given trajectory, with convergence corresponding to $r_p < \epsilon_p$.}
    \label{fig:motion-plan-partial-solutions}
\end{figure}

\begin{figure*}[t]
    \centering
    \input{figs/motion_planning_fp_heuristic_comp.pgf}
    \caption{Times to find a feasible solution, peak memory usage, and percentage of cases where a feasible solution was found for 100 randomly generated motion planning problem. Whiskers in (a) and bands in (b) enclose all data points.}
    \label{fig:motion-plan-sol-time-comp}
\end{figure*}

We next consider the example of a reach-avoid problem formulated using hybrid zonotope reachability analysis. 
\revmargin{1}{R1.4}\revres{1}{This is a canonical robotics planning problem that is fundamentally hybrid in nature and serves as a good basis for comparing mixed-integer solvers~\cite{ioan2021mixed}}
The problem here is to find a feasible trajectory for a discrete time linear system from an initial state $\mathbf{x}_0$ to a terminal set $\mathcal{F}_N$ subject to polytopic state and input domains $\mathcal{S}$ and $\mathcal{U}$ and position constraints $\mathcal{P}$ where $\mathcal{P}$ is a union of polytopes. The horizon of the planning problem is $N$.

The linear dynamics $\mathbf{x}_{k+1} = A \mathbf{x}_k + B \mathbf{u}_k$ of the system are given by the discrete time double integrator model,
\begin{subequations} \label{eq:dbl-int-dyn}
\begin{align}
    &\mathbf{x}_k = \begin{bmatrix} x_k & y_k & v_{x,k} & v_{y,k} \end{bmatrix}^T \;, \\
    &\mathbf{u}_k = \begin{bmatrix} a_{x,k} & a_{y,k} \end{bmatrix}^T \;, \\
    &A = \begin{bmatrix}
        1 & 0 & \Delta t & 0 \\
        0 & 1 & 0 & \Delta t \\
        0 & 0 & 1 & 0 \\
        0 & 0 & 0 & 1
    \end{bmatrix} \;,\; B = \begin{bmatrix}
        \frac{\Delta t^2}{2} & 0 \\
        0 & \frac{\Delta t^2}{2} \\
        \Delta t & 0 \\
        0 & \Delta t
    \end{bmatrix} \;,
\end{align}
\end{subequations}
where $x$ and $y$ are position coordinates, and $v_x$, $v_y$ and $a_x$, $a_y$ are the corresponding velocities and accelerations. The units of $\mathbf{x}_k$ and $\mathbf{u}_k$ are $\begin{bmatrix} \text{m} & \text{m} & \text{m/s} & \text{m/s} \end{bmatrix}^T$ and $\begin{bmatrix} \text{m/s}^2 & \text{m/s}^2 \end{bmatrix}^T$, respectively. 

Because \eqref{eq:dbl-int-dyn} is a linear system, the reachability calculations given for PWA systems in Sec.~\ref{sec:hz_reach} simplify, and we use~\cite[Alg.~2]{robbins2025sparsity}. The obstacle avoidance constraints are enforced by adding the step
\begin{equation}
    \mathcal{Z}_k \gets \mathcal{Z}_k \cap_{\begin{bmatrix} 0 & 0 & \cdots & I & 0 \end{bmatrix}} \mathcal{P} \;,
\end{equation}
to \cite[Algorithm 2]{robbins2025sparsity}, where $\mathcal{Z}_k$ is defined as in Def.~\ref{def:Z_def} and $\mathcal{P}$ is a hybrid zonotope.

In this example, we consider 100 different random obstacle-free spaces $\mathcal{P}$. To construct these random obstacle-free spaces, we first generate three random, five-sided, non-overlapping polytopic obstacles within the space $\mathcal{S}^p = \left\{ \begin{bmatrix} x & y \end{bmatrix}^T \revmargin{3}{R3.2}~\middle|~ x\in [0,10],\; y \in [-5,5] \right\}$. 
Each obstacle has a randomly generated center and vertices with a shortest distance to the center of between 1 and 2~m.
From a vertex representation (V-Rep) of the obstacles and $\mathcal{S}^p$, a free space partition in the form of a union of V-Rep polytopes is constructed using the Hertel and Mehlhorn algorithm~\cite{o1998computational}. This free space partition is then transformed into a hybrid zonotope using~\cite[Thm. 5]{siefert2025reachability}. As shown in~\cite{robbins2024mixed}, hybrid zonotopes constructed in this way are sharp.

The state domain set is $\mathcal{S} = \mathcal{S}^p \times \mathcal{S}^v$, where $\mathcal{S}^v = \mathcal{O}(1,4)$.
The input constraint set is $U = \mathcal{O}(0.1 \cdot \frac{\pi}{2}, 4)$. Zonotopes $\mathcal{S}^v$ and $\mathcal{U}$ inner-approximate quadratic constraints $v_x^2 + v_y^2 \leq v_{\mathrm{max}}^2$ and $a_x^2 + a_y^2 \leq (v_{\mathrm{min}}\omega_{\mathrm{max}})^2$, where $v_{\mathrm{max}}$ and $\omega_{\mathrm{max}}$ are linear and angular velocity limits, respectively, and $v_{\mathrm{min}}$ is the speed above which $\omega_{\mathrm{max}}$ is enforced.

The initial condition for the reach-avoid problem is $\mathbf{x}_0 = \begin{bmatrix} 0.1 & 0 & 0.1 & 0 \end{bmatrix}^T$. The reference state is $\mathbf{x}^r_k = \begin{bmatrix} 10 & 0 & 0 & 0 \end{bmatrix},\; \forall k \in \{1, ..., N\}$. The terminal state constraint set is $\mathcal{F}_N = \{\mathbf{x}^r_N\} \oplus \left( \mathcal{O}(1,6) \times \mathcal{O}(0.01,6)\right)$. The cost function matrices are $Q = \mathrm{diag}([0.1/N, 0.1/N, 0, 0])$, $Q_N = \mathrm{diag}([1, 1, 0, 0])$, and $R = (10/N) I$. To evaluate the scalability of the proposed ADMM-FP heuristic, we define a variable sampling factor $f_s$ such that $\Delta t = 2/f_s$~sec and $N = 10 f_s$. The ADMM-FP, ADMM, and OFP parameters used in this example are the same as those used in Sec.~\ref{sec:random-milp}. Computation times include the time to solve the convex relaxation as is necessary to compute the initial iterates $\bm{\zeta}^*, \mathbf{u}^*$.

Example motion plans for a representative reach-avoid problem are shown in Fig.~\ref{fig:motion-plan-traj}. In this example, Algorithm~\ref{alg:admm-fp} was executed 25 times, and in all 25 cases a feasible solution was found. \revmargin{1}{R1.7}\revres{1}{Solution time and sub-optimality statistics are provided in Tab.~\ref{tab:reach-avoid-stats}.}

\begin{table}[h]
    \centering
    \revres{1}{
    \caption{ADMM-FP solution statistics over 25 trials for a representative reach-avoid problem.}
    \begin{tabular}{l|c c}
        \toprule
         \textbf{statistic} & \textbf{solve time (sec)} & \textbf{sub-optimality (\%)} \\ \midrule
         median & 0.23 & 19.1 \\
         minimum & 0.06 & 3.51 \\
         maximum & 1.29 & 31.6 \\
         \bottomrule
    \end{tabular}
    \label{tab:reach-avoid-stats}}
\end{table}

\revmargin{1}{R1.9}\revres{1}{Fig.~\ref{fig:motion-plan-partial-solutions} illustrates the behavior of the ADMM-FP heuristic prior to convergence. During early iterations, the algorithm explores trajectories that go below the leftmost obstacle. For these iterations, $r_p > \epsilon_p$, and the trajectories noticeably violate obstacle avoidance constraints on the right side of the figure. During later iterations, the algorithm explores trajectories that go around the top of the leftmost obstacle. This ability to escape regions of the search space where progress is stalling is afforded by the proposed perturbation and restart procedures. Fig.~\ref{fig:motion-plan-partial-solutions} also shows that the trajectories at early iterations are well-behaved despite not being strictly feasible. In cases where ADMM-FP does not converge within a given time or iteration budget, the returned trajectory may still be usable in some contexts.}

Fig.~\ref{fig:motion-plan-sol-time-comp} assesses the performance and reliability of the ADMM-FP, OFP, and ADMM heuristics over 100 random obstacle-free spaces $\mathcal{P}$. Results for Gurobi \revmargin{1}{R1.1}\revres{1}{and SCIP}, which are not purely heuristic solvers, are again provided for reference. The scaling factor $f_s$ is used to empirically assess scalability. Computation time was limited to 30~sec for each trial. For the OFP, which relies on iterative LP solutions, the quadratic cost matrix was taken to be $P=0$.

As seen in Fig.~\ref{fig:motion-plan-sol-time-comp}(c), when compared to the ADMM and OFP heuristics, the proposed ADMM-FP heuristic is much more reliable for this task, and is able to find a feasible solution in the vast majority of cases. The ADMM and OFP heuristics, by comparison, quickly become unreliable, with the OFP unable to find a feasible solution in less than 30~sec for $N \geq 30$.
Fig.~\ref{fig:motion-plan-sol-time-comp}(a) shows that the computation time for converged cases is significantly improved using the ADMM-FP heuristic when compared to the OFP. While convergence times for ADMM-FP appear slightly slower than for ADMM in Fig.~\ref{fig:motion-plan-sol-time-comp}(a), this results primarily from ADMM-FP solving more problems than ADMM. For the problems that ADMM solved, the ADMM-FP solution times were similar on average. Defining $t_{\mathrm{ADMM-FP}}$ to be the ADMM-FP solve time and $t_{\mathrm{ADMM}}$ to be the ADMM solve time, for problems where ADMM found a solution, the median $t_{\mathrm{ADMM-FP}} / t_{\mathrm{ADMM}}$ was 1.08 (minimum: 0.11, maximum: 14.19). As such, the solution time penalties for the proposed perturbation operations in the ADMM-FP algorithm are mild and easily justified by the significantly improved convergence properties.

\revmargin{1}{R1.5}\revres{1}{As another benchmark, we consider a simple round-and-solve variant of the ADMM heuristic. In this heuristic, the ADMM algorithm runs for $1000$ iterations, then binary variables are fixed to their rounded values and the corresponding QP is solved. This heuristic produced a feasible solution in $7 / 100$ cases for $N=10$, and in $0$ cases for $N > 10$.}

While Gurobi and SCIP are global MIP solvers and not directly comparable to the mixed-integer heuristics considered here, we note that the ADMM-FP heuristic typically finds feasible solutions to the reach-avoid problem considerably faster than either of these solvers. Further, the ADMM-FP heuristic is only slightly less reliable than Gurobi and SCIP, and finds a feasible solution within the 30~sec time limit in 100\% of trials with $N < 50$. \revmargin{1}{R1.1}\revres{1}{In Fig.~\ref{fig:motion-plan-sol-time-comp}(b), the ADMM and ADMM-FP solution approaches can be seen to use an order of magnitude less memory and to have much more consistent memory usage than general-purpose solvers. The small memory variability in ADMM and ADMM-FP solves is due to variability in problem size and structure within the 100 random reach-avoid scenarios.}
Clearly, for some problem classes such as this reach-avoid example, the ADMM-FP heuristic is competitive with state-of-the-art solvers \revmargin[true]{1}{R1.1}\revres{1}{at a much lower memory footprint} despite being an iterative, heuristic method.

\subsection{Ball and Paddle System} \label{sec:ball-paddle}

In this example, we consider the ball and paddle system benchmark from~\cite{marcucci2019mixed} (also used in~\cite{wu2023soy}). 
\revmargin{1}{R1.4}\revres{1}{This serves as a challenge problem, as solutions are known to be very difficult to identify.}
The ball and paddle system models the interaction of three objects: A fixed ceiling, a movable paddle, and a ball. There are 10 states: The 2D positions of the ball and paddle, the orientation of the ball, and their associated time derivatives. The control inputs are the paddle acceleration components. This is a discrete time PWA system with seven dynamic modes. There is one mode for the case of no contact between the ball and the paddle or ceiling; there are four modes for the ball sliding right or left on the ceiling or the paddle; and there are two modes for the ball sticking/rolling on the ceiling or paddle. The ball and paddle positions and velocities are initially zero. The goal of this planning problem is to find a trajectory of the system that rotates the ball by $180^o$ in 20 time steps while returning the ball and paddle positions and velocities to zero. The ball and paddle example is known to be challenging as it is severely underactuated, with feasible solutions very difficult to identify. See~\cite{marcucci2019mixed} for more information about the ball and paddle system.

We formulate the ball and paddle system planning problem using the hybrid zonotope reachability calculations of Sec.~\ref{sec:hz_reach}. The condensed union identity (App.~\ref{sec:condensed-union}) is used in \eqref{eq:union_of_gofs}.

The proposed ADMM-FP heuristic failed to produce a solution with feasibility tolerance of $\epsilon_p = 0.001$ in less than 30~sec in our testing. However, it consistently produced solutions with $\epsilon_p = 0.01$. \revmargin{1}{R1.8}\revres{1}{A more interpretable numerical solution tolerance is given by projecting the factor space tolerance $\epsilon_p$ onto the problem space as $\mathbf{e} = |G| (\epsilon_p \mathbf{1})$, where $G = [G_c~G_b]$ is the combined hybrid zonotope generator matrix. For this problem, the maximum position component of $\mathbf{e}$ for $\epsilon_p=0.01$ is $0.003~$m. This quantity can be misleading however, as it does not necessarily indicate that there is a strictly feasible solution within $\mathbf{e}$ of the ADMM-FP solution. Instead, this is effectively the resolution of the solve, and says that the ADMM-FP iterates are within $\mathbf{e}$ of each other when projected onto the problem space.}

In 50 trials, 50 \revmargin{2}{R2.11}\revres{2}{solutions within the loosened numerical tolerance of $\epsilon_p=0.01$} were produced. 
\revres{1}{Statistics for these 50 trials are given in Tab.~\ref{tab:ball-paddle-stats}.}\revmargin{1}{R1.7}
\begin{table}[h]
    \centering
    \revres{1}{
    \caption{ADMM-FP solution statistics over 50 trials with $\epsilon_p = 0.01$ for the ball and paddle system.}
    \begin{tabular}{l|c c}
        \toprule
         \textbf{statistic} & \textbf{run time (sec)} & \textbf{iteration count} \\ \midrule
         median & 0.99 & 14{,}915 \\
         minimum & 0.03 & 377 \\
         maximum & 7.8 & 121{,}561 \\
         \bottomrule
    \end{tabular}
    \label{tab:ball-paddle-stats}}
\end{table}

Fig.~\ref{fig:bouncing-ball-traj} shows the solution from one of the trials. In this figure, the planned trajectory can be seen to deviate from a trajectory generated via open-loop simulation using the planned control inputs. 
Nevertheless, the trajectory produced by the ADMM-FP heuristic nearly satisfies the planning problem specifications. In other contexts, such as those where a feedback controller tracks the planned trajectory, a nearly feasible trajectory such as this may still be useful.

\begin{figure}[t]
    \centering
    \input{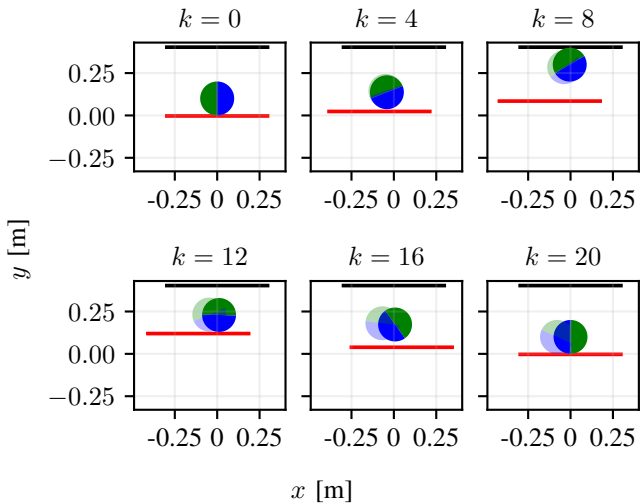}
    \caption{Example ball and paddle system solution with $\epsilon_p = 0.01$. The black bar is the ceiling and the red bar is the paddle. The opaque ball corresponds to the planned trajectory while the translucent ball corresponds to the open-loop simulation using the planned control inputs.}
    \label{fig:bouncing-ball-traj}
\end{figure}

\subsection{Combined Behavior and Motion Planning} \label{sec:combined-behavior-motion-planning}
As the final numerical example, we consider the problem of combined behavior and motion planning for an autonomous vehicle. 
\revmargin{1}{R1.4}\revres{1}{This problem exercises the full PWA system reachability and ADMM-FP solution framework proposed in this article.}
In this scenario, an ego vehicle must plan a collision-free trajectory through a two-lane road where the positions of other obstacle vehicles on the road are time-varying. The system has two dynamic modes (i.e., behaviors): One for left lane tracking and one for right lane tracking.
Note that Sec.~\ref{sec:experiment} has an experimental application of the same system using small ground robots, and all parameter values in this section match those for the experiment.

The ego vehicle dynamics are modeled using the double integrator system~\eqref{eq:dbl-int-dyn}. For the combined behavior and motion planning scenario, we replace coordinates $x$ and $y$ in~\eqref{eq:dbl-int-dyn} with road-aligned Frenet frame coordinates $s$ (distance along the road center line) and $d$ (perpendicular to road center line, positive in right lane), respectively~\cite{reiter2021parameterization}. 

A lane-tracking controller of the form
\begin{equation} \label{eq:comb-beh-motion-fbk-control-law}
    \ddot{d} = -\begin{bmatrix} k_d & k_{\dot{d}}\end{bmatrix} \begin{bmatrix} d \\ \dot{d} \end{bmatrix} + k_{\mathit{ff}} + \ddot{d}' \;,
\end{equation}
is assumed for lateral vehicle control, where $k_d=0.213$~1/s$^2$ 
and $k_{\dot{d}}=0.653$~1/s are feedback terms, $k_{\mathit{ff}}$ is a feedforward term, and $\ddot{d}'$ is a small perturbation. Defining 
\begin{equation}
K = \begin{bmatrix} 0 & 0 & 0 & 0 \\
                    0 & k_d & 0 & k_{\dot{d}}
\end{bmatrix} \;,
\end{equation}
the affine dynamics for right and left lane tracking are
\begin{equation} \label{eq:closed-loop-lane-tracking-modes}
    \mathbf{x}_{k+1} = \begin{cases} A^{\mathit{rl}} \mathbf{x}_k + B^{\mathit{rl}} \mathbf{u}'_k + \mathbf{f}^{\mathit{rl}},& \text{right lane tracking,} \\
    A^{\mathit{ll}} \mathbf{x}_k + B^{\mathit{ll}} \mathbf{u}'_k + \mathbf{f}^{\mathit{ll}},& \text{left lane tracking,}
 \end{cases} 
\end{equation}    
where 
\begin{subequations}
\begin{align}
    &\mathbf{u}'_k = \begin{bmatrix} \ddot{s} & \ddot{d}' \end{bmatrix} \;, \\
    &A^{\mathit{rl}} = A^{\mathit{ll}} = A-BK \;, \\
    &B^{\mathit{rl}} = B^{\mathit{ll}} = B \;, \\
    &\mathbf{f}^{\mathit{rl}} = B K \begin{bmatrix} 0 & d^{\mathit{rl}} & 0 & 0 \end{bmatrix}^T \;, \\
    &\mathbf{f}^{\mathit{ll}} = B K \begin{bmatrix} 0 & d^{\mathit{ll}} & 0 & 0 \end{bmatrix}^T \;,
\end{align}
\end{subequations}
and $d^{\mathit{rl}}$ and $d^{\mathit{ll}}$ are the lateral positions of the center lines of the right and left lanes, respectively. Note that $\mathbf{f}_{\mathit{rl}}$ and $\mathbf{f}_{\mathit{ll}}$ correspond to the $k_{\mathit{ff}}$ term in~\eqref{eq:comb-beh-motion-fbk-control-law}.

\begin{figure}
    \centering
    \input{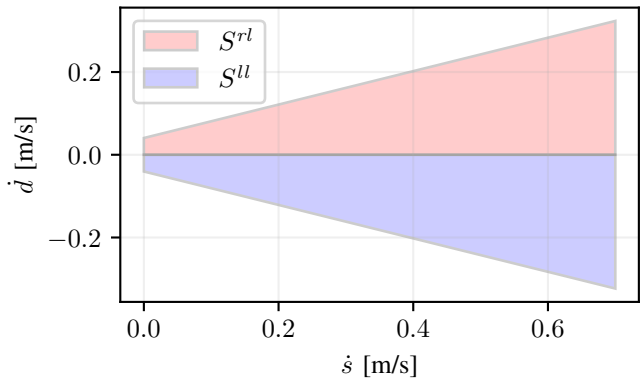}
    \caption{Velocity constraints for combined behavior and motion planning system.}
    \label{fig:behavior-motion-planning-constraints}
\end{figure}

The domains $\revmargin{2}{R2.7}\revres{2}{\mathcal{D}}^i_k$ of the right and left lane tracking modes differ in terms of the velocity constraints, as depicted in Fig.~\ref{fig:behavior-motion-planning-constraints}. The system tracks the right lane when $\dot{d} \geq 0$ and the left lane when $\dot{d} \leq 0$. Position constraints in $\revmargin[true]{2}{R2.7}\revres{2}{\mathcal{D}}^i_k$ bound the relevant segment of the road, i.e., $s \in [0, 10.5]$~m, $d \in [-0.51, 0.51]$~m.
The control input $\ddot{d}'$ provides a mechanism to switch between right and left lane tracking modes. In this example, little authority is given to this control input, with $\ddot{d}' \in [-0.01, 0.01]$~m/s$^2$, to prevent rapid switching between lane tracking modes. The longitudinal acceleration is constrained as $\ddot{s} \in [-1, 1]$~m/s$^2$.

The coupled $\dot{s}$ and $\dot{d}$ constraints prevent lane changes at low vehicle speeds. Here, the constraints are defined such that the heading angle with respect to the road center line does not exceed $30^\circ$ during a turn.

The planning problem for this PWA system is formulated using the hybrid zonotope reachability calculations described in Sec.~\ref{sec:hz_reach}. When computing the graph of the function $\Psi_k$ using~\eqref{eq:union_of_gofs}, the condensed union identity (App.~\ref{sec:condensed-union}) is used.
Obstacle avoidance constraints are enforced using constrained graphs of functions, i.e., $\tilde{\Psi}_k = \Psi_k \cap_{\begin{bmatrix} 0 & 0 & I \end{bmatrix}} (\mathcal{P}_{k+1} \times \overline{S}_v)$, where $\mathcal{P}_k$ is the obstacle-free space at time step $k$ and $\overline{S}_v$ is an interval bound on the velocity $\begin{bmatrix} \dot{s} & \dot{d} \end{bmatrix}^T$.
To construct $\mathcal{P}_k$, the portion of each lane occupied by an obstacle vehicle is removed, bloating to account for ego vehicle geometry and inter-sample behavior. The remaining lane segments are represented as zonotopes $\mathcal{PZ}^i_k$, and $\mathcal{P}_k = \bigcup_i \mathcal{PZ}^i_k$ using the zonotope union identity (App.~\ref{sec:zonotope-union}).

The reference trajectory $\mathbf{x}^r_k$ follows the right lane from the ego vehicle initial $s_0$ at a speed of 0.5~m/s. The cost function matrices are $Q = \mathrm{diag}([0.5, 0.5, 0, 0])$, $Q_N = 10 I$, and $R = 10 I$. The discrete time step is $\Delta t = 1$~sec, and the prediction horizon is $N=15$.

In this example, we deviate from the default ADMM-FP parameters provided in Tab.~\ref{tab:admm-fp-default-params}. Specifically, we choose $\rho = 100$, $\epsilon_p = 0.01$, $k_{\mathrm{restart}}=1000$, and $k_{ph1}=5000$. These parameters were tuned empirically, with $\rho=100$ and $k_{\mathrm{restart}}=1000$ generally showing reduced time to find a low-cost feasible solution for this system when compared to the default parameters. 
For this system, trajectories with $\epsilon_p = 0.01$ are sufficiently close to feasible that choosing a smaller $\epsilon_p$ is not necessary. The dual convergence tolerance for convex relaxation and warm-start QPs was set to $\epsilon_d = 0.1$.

The combined behavior and motion plan for a representative scenario with two obstacle vehicles 
is shown in Fig.~\ref{fig:behavior-motion-planning-scenario}. The ego vehicle changes lanes to pass a slow-moving vehicle in the right lane, and waits for a faster-moving vehicle in the left lane to advance before merging back into the right lane.

\begin{figure*}
\centering
\input{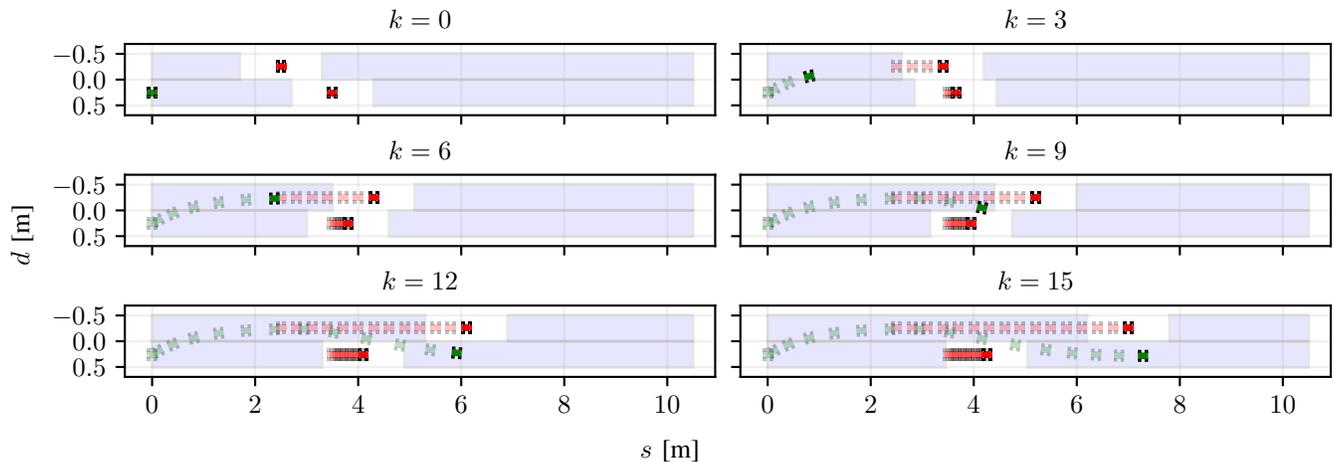}
\caption{Behavior and motion planning scenario. The green vehicle is the ego vehicle, and the red vehicles are the obstacle vehicles. The blue regions are the the time-varying obstacle-free space $\mathcal{P}_k$.}
\label{fig:behavior-motion-planning-scenario}
\end{figure*}

To empirically evaluate the efficacy of the warm-start procedure proposed in Sec.~\ref{sec:warm-start}, we applied random perturbations to feasible trajectories $\mathbf{z} \in \mathcal{Z}_N$ generated using the default convex relaxation procedure for producing initial iterates $\bm{\zeta}^*, \mathbf{u}^*$. The warm-start vector was constructed as $\mathbf{z}^* = \mathbf{z} + \mathbf{w}$ where the elements $w_i$ of $\mathbf{w}$ are drawn from a normal distribution with $\mu=0$ and $\sigma=0.1$.

\begin{figure}
\centering
\input{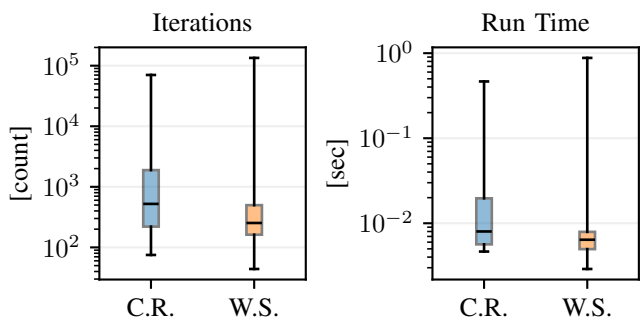}
\caption{Comparison of convex relaxation (C.R.) initialization and warm-start (W.S.) initialization for randomly generated combined behavior and motion planning scenarios.
Warm-start times include the time to solve~\eqref{eq:warmstart-proj}, and convex relaxation times include the time to solve the relaxation of~\eqref{eq:optim_prob_gen}. Whiskers enclose all data points.}
\label{fig:warmstart-stats}
\end{figure}

Fig.~\ref{fig:warmstart-stats} shows the number of iterations and run time over 100 randomly generated traffic scenarios. The random scenarios differ from Fig.~\ref{fig:behavior-motion-planning-scenario} only in the initial positions and constant speeds of the two obstacle vehicles. Initial positions were randomly selected from a uniform distribution over the range $[0, v_{\mathrm{max}} \cdot \Delta t \cdot N]$, where $0$ is the initial position of the ego vehicle. The obstacle vehicle speeds were drawn from a normal distribution with a mean of 0.2~m/s and a standard deviation of 0.1~m/s.

The ADMM-FP heuristic with the default convex relaxation-based initialization produced a feasible solution within 1~sec in 90 of the 100 random scenarios. The warm-start procedure produced a feasible solution in 89 of 90 scenarios.

On average, for the combined behavior and motion planning system, the warm start procedure results in reduced iterations and run time to produce a feasible solution when compared to cases using the convex relaxation of the original optimization problem~\eqref{eq:optim_prob_gen}.

\subsection{Summary and Note to Practitioners}
In general, mixed-integer programming heuristics are not guaranteed to produce feasible solutions. However, for some problems, they can quickly produce good quality, feasible solutions without requiring memory-intensive, branch-and-bound-based methods.
For the random MILP example in Sec.~\ref{sec:random-milp}, all of the tested heuristics performed well. For the reach-avoid problem in Sec.~\ref{sec:reach-avoid}, the performance and reliability of the OFP and ADMM heuristics drops off precipitously as the planning horizon increases, while the proposed ADMM-FP heuristic continues to perform well. The ADMM-FP heuristic performs well for the combined behavior and motion planning example of Sec.~\ref{sec:combined-behavior-motion-planning} as well, but struggles to produce highly accurate solutions to the challenging ball and paddle system problem of Sec.~\ref{sec:ball-paddle}.

The proposed ADMM-FP heuristic extends the class of problems to which mixed-integer programming heuristics can be effectively applied, but it may not be suitable for every problem as seen in Sec.~\ref{sec:ball-paddle}. In practice, we find that the ADMM-FP heuristic performs best when 1) there are many feasible combinations of binary variables in the planning problem, and 2) the initial iterates $\bm{\zeta}^*, \mathbf{u}^*$ are not too far from a feasible solution, either via effective warm-starting or strong convex relaxation.

\section{Experimental Results} \label{sec:experiment}

In this section, we experimentally evaluate the efficacy and efficiency of the proposed methods in application to combined behavior and motion planning for an autonomous vehicle, 
This uses the same task, dynamics model, and parameters as in Sec.~\ref{sec:combined-behavior-motion-planning}.

This experiment was implemented in a motion capture robotics laboratory using small autonomous ground robots.
OptiTrack Prime$^\text{X}$22 motion capture cameras were used for \revmargin{2}{R2.10}\revres{2}{online} pose and velocity feedback. A two-lane circuit was constructed inside the motion capture volume, and two parked obstacle vehicles were placed in the right lane. A Clearpath TurtleBot 4 was used as a moving obstacle vehicle. The TurtleBot was programmed to drive in the left lane at a constant Frenet frame speed of $\dot{s}=0.2$~m/s. The ego vehicle was a Husarion ROSbot 3, which has a Raspberry Pi 5 embedded computer. All combined behavior and motion planning calculations 
were implemented as C++ ROS2 nodes and executed onboard the ROSbot. These include the hybrid zonotope reachability analysis and ADMM-FP heuristic.

The combined behavior and motion planning calculations, as described in Sec.~\ref{sec:combined-behavior-motion-planning}, are carried out in the Frenet frame.
The ROSbot 3 is a differential drive robot and its dynamics are accurately described by the unicycle model which, in the Frenet frame, is
\begin{equation} \label{eq:frenet-unicycle}
    \dot{s} = \frac{v \cos(\mu)}{1 - \kappa(s) d},~\dot{d} = v \sin(\mu),~\dot{\mu} = \omega - \kappa(s) \dot{s},~\dot{v} = a \;,
\end{equation}
where $s$ and $d$ are position coordinates, $\mu$ is the heading angle relative to the road centerline, and $v$ is the velocity. The forward acceleration $a$ and angular velocity $\omega$ are inputs to the system, and $\kappa(s)$ is the road curvature.
Eq.~\eqref{eq:frenet-unicycle} is differentially flat with flat outputs $s$ and $d$, i.e.,
\begin{subequations}
\begin{align}
    &\mu = \arctan \left( \frac{\dot{d}}{\dot{s}(1 - \kappa(s) d)} \right) \;, \\
    &v = \sqrt{\left( \dot{s} (1 - \kappa(s) d) \right)^2 + \dot{d}^2} \;.
\end{align}
\end{subequations}
This permits planning in terms of $s$, $d$ and their derivatives, thereby justifying the use of the double integrator model formulation described in Sec.~\ref{sec:combined-behavior-motion-planning}.
Similar equations to~\eqref{eq:frenet-unicycle} for the kinematic single track model, which is more commonly used in autonomous vehicle motion planning, are given in~\cite{reiter2021parameterization}. The single track model is also differentially flat. 

In our experiment, we transform Frenet frame trajectories into Cartesian frame trajectories with position coordinates $x$ and $y$, heading $\theta$, and speed $v$. 
A path tracking controller is used to track these trajectories, compensating for model error and disturbances.

We implemented a default behavior and motion plan for trivial scenarios. This plan extends the active behavior and motion plan, propagating the system state forward using~\eqref{eq:closed-loop-lane-tracking-modes}, with the lane tracking mode set to maintain the mode that is active at the end of the planning horizon. The speed component $\dot{s}$ is similarly regulated to $\dot{s}^r =0.5$~m/s, assuming the control law $\ddot{s} = -k_{\dot{s}}(\dot{s} - \dot{s}^r)$ with $k_{\dot{s}} = 0.618$~s$^{-1}$.

When the default behavior and motion plan was predicted to result in a collision, a ROS2 service implementing hybrid zonotope reachability analysis (Algorithm~\ref{alg:online_planning_problem_formulation}) and ADMM-FP (Algorithm~\ref{alg:admm-fp}) calculations was called. 
Inputs to the service include constrained graphs of functions $\tilde{\Psi}_k$, which are calculated using \eqref{eq:union_of_gofs} and Eq.~\eqref{eq:pwa_single_mode_GOF_implementation}, the initial condition $\mathbf{x}_0$, the state and input bounds $\overline{\mathcal{S}}$ and $\overline{\mathcal{U}}$, reference trajectory $\mathbf{x}^r_k$, and cost function matrices $Q$, $R$, and $Q_N$. The warm-start procedure of Sec.~\ref{sec:warm-start} was used, with warm-start vector $\mathbf{z}^*$ taken to be the state and input sequence as defined in Def.~\ref{def:Z_def} for the active behavior and motion plan.
The ADMM-FP solution time was limited to 1~sec. 

When constructing the time-varying $\tilde{\Psi}_k$, obstacle vehicles were assumed to drive at constant $\dot{s}$ and maintain their current lane. The motion capture system was used for obstacle vehicle position and velocity feedback. 

\begin{figure*}[t]
    \centering
    \input{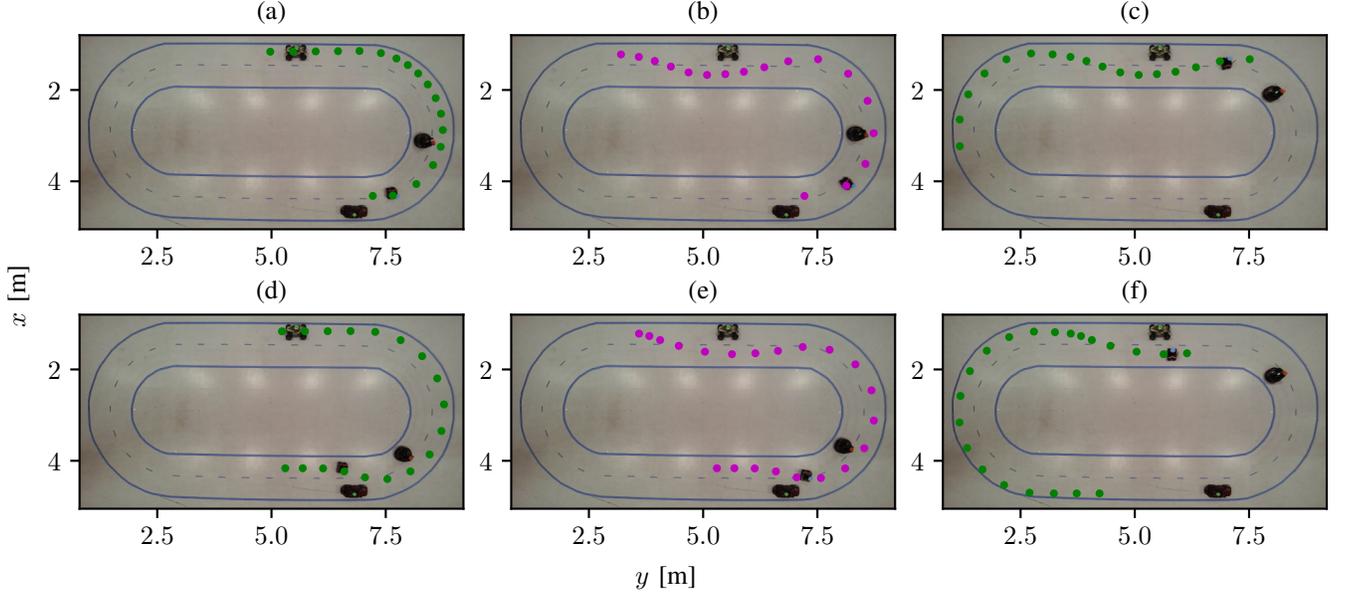}
    \caption{Overhead images from the experimental demonstration of combined behavior and motion planning. Imperfect correction of the fisheye effect in the overhead camera results in some minor misalignment of the behavior and motion plan with the camera images.}
    \label{fig:combined-behavior-motion-planning-experiment-traj}
\end{figure*}

Over the course of this experiment, the service implementing Algorithms~\ref{alg:online_planning_problem_formulation} and \ref{alg:admm-fp} was invoked 19 times. In 15 of these cases, the ADMM-FP algorithm was able to find a feasible solution within the maximum allowed time of 1~sec. 
When the algorithm did not find a solution, the vehicle continued to follow the last valid motion plan and the planning service was invoked again. 
The 4 times this occurred consisted of 2 occasions where the algorithm did not find a solution for two consecutive invocations of the planning service.
Because the planning horizon was 15 seconds, the time to collision was never less than about 13 seconds. Given the wide variability in solution times for MIP solvers, and lack of convergence guarantees for heuristic algorithms, safety must be ensured without requiring the \revmargin{2}{R2.10}\revres{2}{online} solution of an MIP. In our implementation, we utilize a simple emergency braking procedure if the time to collision falls below 3 seconds.

Two example combined behavior and motion plan updates using Algorithms~\ref{alg:online_planning_problem_formulation} and \ref{alg:admm-fp} are depicted in Fig.~\ref{fig:combined-behavior-motion-planning-experiment-traj}. Sub-plots (a)-(c) show one planning sequence and (d)-(f) show another. Two parked obstacle vehicles are in the right lane at the top and bottom of each image, and a third, moving obstacle vehicle is seen in the left lane. The fourth vehicle is the ego vehicle. The combined behavior and motion plan is depicted with a sequence of dots -- green when updated with the default algorithm and magenta when updated with Algorithms~\ref{alg:online_planning_problem_formulation} and \ref{alg:admm-fp}.
In (a) and (d), the default behavior and motion plan is in collision with the parked obstacle vehicle at the top of the image, and the behavior and motion planning service is invoked. In (b) and (e), the combined behavior and motion plan is updated using Algorithms~\ref{alg:online_planning_problem_formulation} and \ref{alg:admm-fp}. The evolution of the scenario after several seconds is shown in (c) and (f), and the ego vehicle is seen to be safely navigating around other vehicles in the environment according to the combined behavior and motion plan. Extensions to the plan in (c) and (f) relative to (b) and (e) are products of the default planning algorithm. 

Table~\ref{tab:experiment-stats} gives problem formulation and solution time statistics for the 15 cases where the ADMM-FP algorithm found a feasible solution within the specified time limit. The problem formulation times using hybrid zonotope reachability analysis are consistently on the order of milliseconds. The time to solve the warm-start QP is typically on the order of tens of milliseconds.
The range of solution times for the ADMM-FP heuristic spans multiple orders of magnitude, in keeping with Figs.~\ref{fig:random-milp}, \ref{fig:motion-plan-sol-time-comp}, and \ref{fig:warmstart-stats}.
Typically, the ADMM-FP heuristic performs quite well, with its median solution time less than that of the warm-start QP. This is not generally possible with conventional branch-and-bound methods as they require the solution of many QPs to generate a feasible solution.

\setlength{\tabcolsep}{7pt}
\begin{table}
    \caption{Problem formulation and solution time statistics for combined behavior and motion planning experiment.}
    \centering
    \begin{tabular}{l|c|c c c}
        \toprule
         & unit & median & min & max \\ \midrule
         ADMM-FP time & ms & 5.2 & 3.4 & 103.0 \\ \midrule
         ADMM-FP iterations & count & 114 & 21 & 3888 \\ \midrule
         Warm start time & ms & 15.4 & 4.1 & 34.5 \\ \midrule
         Eq.~\eqref{eq:union_of_gofs} and Eq.~\eqref{eq:pwa_single_mode_GOF_implementation} time & ms & 1.9 & 1.3 & 1.9 \\ \midrule
         Alg.~\ref{alg:online_planning_problem_formulation} time & ms & 4.4 & 2.9 & 6.4         
         \end{tabular}
    \label{tab:experiment-stats}
\end{table}

\section{Conclusion} \label{sec:conclusion}

This paper presented hybrid zonotope-based reachability calculations for PWA systems and used them to formulate optimal planning problems. A novel MIP heuristic, termed ADMM-FP, was proposed to find feasible solutions to these planning problems.

Experimental results demonstrate that this framework is tractable to implement on embedded systems and can be very efficient; in the experiment, the median time to formulate and solve a combined behavior and motion planning problem was about 25~ms.
There are no convergence guarantees or solution time bounds for the ADMM-FP heuristic, however, and a feasible solution may not be produced within the specified time limit. As such, real-world implementations must be able to handle this possibility. Mitigations for the heuristic failing to produce a solution may include the use of low-level safety filters such as control barrier functions~\cite{ames2019control}. The proposed problem formulation and heuristic may also be used in parallel with other planning approaches (e.g., sampling-based or machine learning-based), with the results only accepted when they produce lower cost plans than those of the other approaches. 

Future work that further improves the convergence properties of the ADMM-FP heuristic would be highly valuable.
Outside of the hybrid zonotope context, the proposed ADMM-FP heuristic may be useful as a primal heuristic for mixed-integer programming, or in other mixed-integer planning frameworks. 

In addition to formulating planning problems, the proposed PWA system reachability calculations may be useful in other applications such as safety verification.
In particular, by reducing the complexity of the hybrid zonotope reachable set representation when compared to MLD system reachability calculations, less frequent application of complexity reduction methods may be needed. 
By improving the tightness of the set representation, the convex relaxation is a less conservative over-approximation and may be useful for complexity reduction.

\appendix

\section{Hybrid Zonotope Union Identities}
This appendix provides three different hybrid zonotope union identities: A sharp union, a condensed union, and a union of zonotopes. Because the following identities are tangential to the main contributions of this paper, they are presented without proof. 

\subsection{Sharp Union} \label{sec:sharp-union}
\newcommand{\dims}{n}

In \cite[Proposition 2]{glunt2025sharp}, an identity for a sharp union of hybrid zonotopes is presented using set operations. Expanding the set operations to matrices produces the following identity.

\begin{proposition}
Given the hybrid zonotopes $\mathcal{Z}_1,\dots,\mathcal{Z}_N \subset \real^{\dims}$, where $\mathcal{Z}_i = \langle G_{c,i}, G_{b,i}, \mathbf{c}_i, A_{c,i}, A_{b,i}, \mathbf{b}_i \rangle_{01}$, the union is given by $\mathcal{Z}=\langle G_c, G_b, \mathbf{c}, A_c, A_b, \mathbf{b} \rangle_{01}$, where
\allowdisplaybreaks
\begin{subequations}
\label{eq:sharp_union_identity}
\begin{align}
    &G_c = \begin{bmatrix}
        G_{c,1} & 0_{\dims\times n_{G,1}}  & G_{c,2} & 0 & \cdots & G_{c,N} & 0
    \end{bmatrix} \;, \\
    &G_b = \begin{bmatrix}
        G_{b,1} & c_1 & G_{b,2} & c_2 & \cdots & G_{b,N} & c_n 
    \end{bmatrix} \;, \\
    &c = \mathbf{0} \;, \\
    &\tilde{A}_{c,i} = \begin{bmatrix}
        A_{c,i} & 0_{n_{C,i}\times n_{G,i}} \\
        \begin{bmatrix} I_{n_{Gc,i}} \\ 0_{n_{Gb,i}\times n_{Gc,i}} \end{bmatrix} & I_{n_{G,i}} \end{bmatrix}\;,\\
    &A_c = \begin{bmatrix}
        \tilde{A}_{c,1} & 0 & \dots & 0 \\
        0 & \tilde{A}_{c,2} & \dots & 0 \\
        \vdots & & \ddots & \\
        0 & 0 & \cdots & \tilde{A}_{c,N} \\
        \mathbf{0}_{2n_{Gc,1}+n_{Gb,1}}^T & \mathbf{0}^T & \cdots & \mathbf{0}^T
    \end{bmatrix} \;, \\
    &\tilde{A}_{b,i} = \begin{bmatrix}
        A_{b,i} & -\mathbf{b}_{i} \\
        \begin{bmatrix} 0_{n_{Gc,i} \times n_{Gb,i}} \\ I_{n_{Gb,i}} \end{bmatrix} & -\mathbf{1}_{n_{G,i}}
    \end{bmatrix} \;, \\
    &A_b \!=\! \begin{bNiceMatrix}
        \Block{1-2}{\tilde{A}_{b,1}} && \Block{1-2}{0} && \dots & \Block{1-2}{0} & \\
        \Block{1-2}{0} && \Block{1-2}{\tilde{A}_{b,2}} && \dots & \Block{1-2}{0} & \\
        \Block{1-2}{\vdots} && && \ddots && \\
        \Block{1-2}{0} && \Block{1-2}{0} && \cdots & \Block{1-2}{\tilde{A}_{b,N}} & \\
        \mathbf{0}_{n_{Gb,1}}^T & 1 & \mathbf{0}^T & 1 & \cdots & \mathbf{0}^T & 1
    \end{bNiceMatrix}, \\
    &\tilde{\mathbf{b}}_i = \mathbf{0}_{n_{G,i} + n_{C,i}} ,\, 
    \mathbf{b} = \begin{bmatrix} \tilde{\mathbf{b}}_1^T & \tilde{\mathbf{b}}_2^T & \cdots & \tilde{\mathbf{b}}_n^T & 1 \end{bmatrix}^T \!\!,
\end{align}
\end{subequations}
and $n_{G,i}=n_{Gc,i}+n_{Gb,i}$. 
\end{proposition}

The set representation complexity of the identity defined by~\eqref{eq:sharp_union_identity} is
$n_{Gc} = \sum_{i=1}^N (2n_{Gc,i}+n_{Gb,i})$, $n_{Gb} = N + \sum_{i=1}^N n_{Gb,i}$, $n_C = 1 + \sum_{i=1}^N ( n_{Gc,i}+n_{Gb,i}+n_{C,i})$.

\subsection{Condensed Union} \label{sec:condensed-union}
\begin{proposition}
The hybrid zonotope union $\bigcup_{i=1}^N \mathcal{Z}_i$, where $\mathcal{Z}_i = \langle G_{c,i}, G_{b,i}, \mathbf{c}_i, A_{c,i}, A_{b,i}, \mathbf{b}_i \rangle_{01}$, 
is given by
$\mathcal{Z} = \langle G_c, G_b, \mathbf{c}, A_c, A_b, \mathbf{b} \rangle_{01}$, where
\allowdisplaybreaks
\begin{subequations} \label{eq:condensed-union}
\begin{align}
    &G_c = \begin{bmatrix}
        G_{c,1} & \mathbf{0}  & G_{c,2} & \mathbf{0} & \cdots & G_{c,N} & \mathbf{0}
    \end{bmatrix} \;, \\
    &G_b = \begin{bmatrix}
        G_{b,1} & \mathbf{c}_1 & G_{b,2} & \mathbf{c}_2 & \cdots & G_{b,N} & \mathbf{c}_n 
    \end{bmatrix} \;, \\
    &\mathbf{c} = \mathbf{0} \;, \\
    &\tilde{A}_{c,i} = \begin{bmatrix}
        \mathbf{1}_{n_{Gc,i}}^T & n_{Gc,i}+n_{Gb,i} \\
        A_{c,i} & \mathbf{0} \end{bmatrix}\;,\\
    &A_c = \begin{bmatrix}
        \tilde{A}_{c,1} & 0 & \dots & 0 \\
        0 & \tilde{A}_{c,2} & \dots & 0 \\
        \vdots & & \ddots & \\
        0 & 0 & \cdots & \tilde{A}_{c,N} \\
        \mathbf{0}_{n_{Gc,1}+1}^T & \mathbf{0}^T & \cdots & \mathbf{0}^T
    \end{bmatrix} \;, \\
    &\tilde{A}_{b,i} = \begin{bmatrix}
        \mathbf{1}_{n_{Gb,i}}^T & -n_{Gc,i}-n_{Gb,i} \\
        A_{b,i} & -\mathbf{b}_i
    \end{bmatrix} \;, \\
    &A_b \!=\! \begin{bNiceMatrix}
        \Block{1-2}{\tilde{A}_{b,1}} && \Block{1-2}{0} && \dots & \Block{1-2}{0} & \\
        \Block{1-2}{0} && \Block{1-2}{\tilde{A}_{b,2}} && \dots & \Block{1-2}{0} & \\
        \Block{1-2}{\vdots} && && \ddots && \\
        \Block{1-2}{0} && \Block{1-2}{0} && \cdots & \Block{1-2}{\tilde{A}_{b,N}} & \\
        \mathbf{0}_{n_{Gb,1}}^T & 1 & \mathbf{0}^T & 1 & \cdots & \mathbf{0}^T & \!\!\!1
    \end{bNiceMatrix}, \\
    &\tilde{\mathbf{b}}_i = \mathbf{0}_{n_{C,i}+1}^T \;,\; \mathbf{b} = \begin{bmatrix} \tilde{\mathbf{b}}_1^T & \tilde{\mathbf{b}}_2^T & \cdots & \tilde{\mathbf{b}}_n^T & 1 \end{bmatrix}^T \;.
\end{align}
\end{subequations}
\end{proposition}

The set representation complexity of the identity defined by~\eqref{eq:condensed-union} is given by
$n_{Gc} = N + \sum_{i=1}^N n_{Gc,i}$, $n_{Gb} = N + \sum_{i=1}^N n_{Gb,i}$, $n_{C} = N + 1 + \sum_{i=1}^N n_{C,i}$.

The condensed union does not in general preserve sharpness, as seen in Fig.~\ref{fig:two_equilibrium_convex_relaxations}, though it often has fairly strong convex relaxations and can be sharp in some special cases.

\subsection{Zonotope Union} \label{sec:zonotope-union}
Consider the union of $N$ zonotopes $\mathcal{Z}_i = \langle G_i, \mathbf{c}_i \rangle_{01}$. The generator matrix for zonotope $i$ can be written in terms of individual generators as $G_i = \begin{bmatrix} \mathbf{g}_{i,1} & \mathbf{g}_{i,2} & \cdots \end{bmatrix}$. 
We define a shared generator matrix $\tilde{G} = \begin{bmatrix} \tilde{\mathbf{g}}_1 & \tilde{\mathbf{g}}_2 & \cdots \end{bmatrix}$ such that, for each individual zonotope generator $\mathbf{g}_{i,k}$, $\exists \tilde{\mathbf{g}}_j~\text{s.t.}~\mathbf{g}_{i,k} = \tilde{\mathbf{g}}_j$. An incidence matrix $M$ is also defined such that 
\begin{equation}
    M_{ji} = \begin{cases}
        1, & \tilde{\mathbf{g}}_j \in G_i, \\
        0, & \text{otherwise}.
    \end{cases} 
\end{equation}
The number of zonotope generator matrices $G_i$ where $\tilde{\mathbf{g}}_j$ appears is denoted $\tilde{N}_j$. 
\begin{proposition} \label{prop:zono-union}
The hybrid zonotope $\mathcal{Z} = \bigcup_i^N \mathcal{Z}_i$ is given as
\begin{equation} \label{eq:zono-union-identity}
    \mathcal{Z} = \left\langle \begin{bmatrix} \tilde{G} & 0 \end{bmatrix}, C, \mathbf{0}, 
    \begin{bmatrix}
        I & \tilde{N} \\
        \mathbf{0}^T & \mathbf{0}^T
    \end{bmatrix}, \begin{bmatrix}
        -M \\
        \mathbf{1}^T
    \end{bmatrix}, \begin{bmatrix}
        \mathbf{0} \\
        1
    \end{bmatrix} \right\rangle_{01} \;,
\end{equation}
where $C = \begin{bmatrix} \mathbf{c}_1 & \mathbf{c}_2 & \cdots \end{bmatrix}$ and $\tilde{N} = \mathrm{diag}([\tilde{N}_1, \tilde{N}_2, \cdots])$.
\end{proposition}

Proposition~\ref{prop:zono-union} has similar structure to the vertex representation polytope union from \cite[Thm. 5]{siefert2024reachability} and generalizes an identity for representing grids as hybrid zonotopes in \cite{robbins2024efficient}.

The complexity of the zonotope union is 
\begin{equation}
    n_{Gc} = 2 n_{\tilde{G}},\;
    n_{Gb} = N,\; 
    n_C = n_{\tilde{G}} + 1 \;,
\end{equation}
where $n_{\tilde{G}}$ is the number of unique generators $\tilde{\mathbf{g}}_j$.
This identity is most efficient for cases where there are many zonotopes with shared generators. A simple example would be the case of a grid, where the generators are shared for each cell in the grid. 

From inspection of~\eqref{eq:sharp_union_identity} and \eqref{eq:zono-union-identity}, the sharp union reduces to the zonotope union when there are no shared generators between the constituent zonotopes $\mathcal{Z}_i$. The hybrid zonotope given in~\eqref{eq:zono-union-identity} is sharp.

\bibliography{bibitems}
\bibliographystyle{ieeetr}

\end{document}